\documentclass[runningheads]{llncs}
\usepackage[misc]{ifsym}
\usepackage{url}\usepackage{graphicx}
\usepackage{mathtools} 
\usepackage{booktabs} 
\usepackage{tikz} 
\usepackage{multirow}
\usepackage{algorithm}
\usepackage[noend]{algpseudocode}
\usepackage[colorinlistoftodos,prependcaption,textsize=small]{todonotes}
\usepackage{bbm}
\usepackage{cite}
\definecolor{bostonuniversityred}{rgb}{0.8, 0.0, 0.0}

\usepackage{filecontents,pgfplots}
\usepgfplotslibrary{colormaps}
\usepackage{pgfplotstable}
\usepgfplotslibrary{groupplots}
\pgfplotsset{compat=newest}
\usepackage{color}
\definecolor{col1}{RGB}{72, 24, 106}
\definecolor{col2}{RGB}{31, 160, 136}
\definecolor{col3}{RGB}{216, 226, 25}
\definecolor{col4}{RGB}{52, 96, 141}

\graphicspath{{./}{../}}

\begin{document}
\title{Scoring rule nets: beyond mean target prediction in multivariate regression}

%
%
\author{Daan Roordink\inst{1}\orcidID{0009-0000-5190-5596} \and
Sibylle Hess\inst{2} (\Letter)\orcidID{0000-0002-2557-4604}}
\tocauthor{Daan Roordink and Sibylle Hess}
\toctitle{Scoring rule nets: beyond mean target prediction in multivariate regression}

\authorrunning{Roordink and Hess}
%
\institute{Enexis Group, 's Hertogenbosch, Netherlands,\\\email{daan.roordink@enexis.nl} \and
Mathematics \& Computer Science Dept.,
   Eindhoven University of Technology,
   Eindhoven, Netherlands\\ \email{s.c.hess@tue.nl}}

\maketitle              
\begin{abstract}
Probabilistic regression models trained with maximum likelihood estimation (MLE), can sometimes overestimate variance to an unacceptable degree. This is mostly problematic in the multivariate domain. While univariate models often optimize the popular Continuous Ranked Probability Score (CRPS), in the multivariate domain, no such alternative to MLE has yet been widely accepted. The Energy Score -- the most investigated alternative -- notoriously lacks closed-form expressions and sensitivity to the correlation between target variables.
In this paper, we propose Conditional CRPS: a multivariate strictly proper scoring rule that extends CRPS. We show that closed-form expressions exist for popular distributions and illustrate their sensitivity to correlation. We then show in a variety of experiments on both synthetic and real data, that Conditional CRPS often outperforms MLE, and produces results comparable to state-of-the-art non-parametric models, such as Distributional Random Forest (DRF).

\end{abstract}

\section{Introduction}\label{sec:intro}
The vanilla regression models predict a single target value $y$ for an observation $\mathbf{x}\in\mathbbm{R}^p$. In theory, the  goal is to approximate the \emph{true} regression model $f^*$, generating the observed target values as samples of the random variable $Y = f^*(\mathbf{x})+\epsilon$. The random variable $\epsilon$ reflects here the noise in the data and is assumed to have an expected value of zero. Hence, the goal is to find a regression model that predicts the mean $f(\mathbf{x}) = \mathbbm{E}_Y[Y\mid \mathbf{x}] = f^*(\mathbf{x})$. 

However, in practice, the trained regression models come with uncertainties. Reflecting those uncertainties is relevant, for example when a lower or upper bound for the prediction is of interest, when underforecasting has more detrimental consequences than overforecasting, or when the expected profit and risk are dependent on prediction uncertainty. Examples of such applications are found in weather forecasting~\cite{zhu}, healthcare~\cite{healthcare}, predictions of the electricity price~\cite{Nowotarski2015}, stock price \cite{stock_market}, survival rate~\cite{surival_rate_prediction} and air quality~\cite{air_quality}. 

Distributional regression models provide predictive uncertainty quantification by modeling the target variable as a probability distribution. That is, models are tasked with predicting the distribution of a (possibly multivariate) random variable $Y$, conditioned on an observation $x$ of a (possibly multivariate) covariate random variable $X$:
\begin{equation}\label{eq:distReg}
    f(x) = P(Y \mid X = x)\text{.}
\end{equation}
Here, $P(\cdot)$ denotes the probability distribution of a random variable. Such a model is trained on a dataset of observations of $(X,Y)$: $\{(\mathbf{x}_i, \mathbf{y}_i)\}_{i=1}^n$.

Distributional regression models are typically trained by Maximum Likelihood Estimation (MLE) \cite{Haynes2013}, which is equivalent to minimizing the Logarithmic Score. However, when the assumed and the true shape of the distribution do not match, MLE can become sensitive to outliers \cite{BJERREGARD2021100058}, causing a  disproportionally increase in the forecasted variance \cite{gebetsberger}. While this is not necessarily a problem for homoskedastic models (where typically only a single estimator is predicted and the error distribution is assumed to be constant), it is problematic for heteroskedastic models, predicting the full distribution~\cite{surival_rate_prediction}. Therefore, for a univariate continuous target domain, many distributional regression approaches use the Continuous Ranked Probability Score (CRPS) \cite{crps_source}. CRPS provides an optimization objective that is generally more robust than MLE \cite{lerchrasp} and hence gains in popularity in comparison to MLE~\cite{lerchrasp, surival_rate_prediction, air_quality}.

However, unlike MLE, CRPS has no extension to the multivariate domain ($\mathbf{y}\in\mathbbm{R}^d$) that maintains the robustness properties. The most popular extension is the Energy Score \cite{gneitingRaftery}, but it is known to be insensitive to correlation, and often cannot be analytically evaluated \cite{pinsontastu}. Moreover, other alternatives such as the Variogram Score \cite{variogram_score} also have weaknesses, such as translational invariance.

The lack of a robust alternative to MLE is widely discussed in comparative studies. In their review of probabilistic forecasting, Gneiting and Katzfuss argue that ``a pressing need is to go beyond the univariate, real-valued case,
which we review, to the multivariate case'' \cite{gneitingKatzfuss}. More recently, Alexander et al. conclude that ``it is rarely seen that one metric for evaluating the accuracy of a forecast consistently outperforms another metric, on every single scenario'' \cite{alexander}. As a result, multivariate distributional regression approaches either resort to MLE \cite{MUSCHINSKI2022} or avoid direct usage of distributional optimization criteria, via approaches based on e.g. Generative Adversarial Networks (GANs) \cite{C-GAN-regression} or Random Forests (RF) \cite{DRF}.

\paragraph{Contributions}
\begin{enumerate}
\item We propose a novel scoring rule for multivariate distributions, called Conditional CRPS (CCRPS). The novel scoring rule CCRPS is a multivariate extension of the popular univariate CRPS that is more sensitive to correlation than the Energy Score and (for some distributions) less sensitive to outliers than the Logarithmic Score. We enable the numerical optimization of the proposed scoring rule by proving equivalent, closed-form expressions for a variety of multivariate distributions, whose gradients are easy to compute.
\item We propose two novel loss functions for Artificial Neural Network-based multivariate distributional regression, with loss functions based on Conditional CRPS, and the Energy Score.
\item We show on a variety of synthetic and real-world case studies that the two proposed methods often outperform current state-of-the-art.
\end{enumerate}

\section{Distributional Regression}
Distributional regression models are generally evaluated via two concepts: sharpness and calibration~\cite{gneiting_sharpness}. Calibration is the notion that predictions should match the statistics of the actual corresponding observations. For example, when predicting a $30\%$ chance of snow, snowfall should indeed occur in $30\%$ of the corresponding observations.  The goal of regression can then be formulated to maximize sharpness (i.e. the precision of the predicted distribution) under calibration~\cite{gneiting_sharpness}. For example, using the notation of the introduction, both models $f(x) = P(Y \mid X = x)$ and $g(x) = P(Y)$ are calibrated, but if there exists a dependency between $X$ and $Y$, then $f$ is arguably sharper. For this purpose, proper scoring rules are often used.

\subsection{Proper Scoring Rules}
\textit{Scoring rules} are a class of metrics $R$ that compare a predicted distribution $P$ with actual observations $y$. A scoring rule is called \textit{proper} for a class of probability distributions $ \mathcal{D}$ if for any $P, Q \in \mathcal{D}$ we have:
\begin{equation}\label{eq:prop_sr}
    \mathbbm{E}_{Y\sim P}[R(P, Y)] \leq \mathbbm{E}_{Y\sim P}[R(Q, Y)].
\end{equation}
That is, in expectation over all observations, a scoring rule attains its minimum if the distribution of the observations $Y\sim P$ matches the predicted distribution. A scoring rule is called \textit{strictly proper} if the minimum of the expected scoring rule is uniquely attained at $P$. 
Proper and strictly proper scoring rules pose valuable loss functions for distributional regression models: minimizing the mean scoring rule automatically calibrates the model's predicted distributions, and fits the conditional distributions to the observed data (Equation \eqref{eq:distReg}), arguably maximizing sharpness~\cite{gneitingKatzfuss,lerchrasp}.

For univariate domains, the most popular scoring rules are the Logarithmic Score and the Continuous Ranked Probability Score (CRPS). The Logarithmic Score maximizes the MLE criterion, and is defined as
\begin{equation}
    \mathrm{LogS}(P,y) = - \log f_P(y)
\end{equation}
where $f_P$ is $P$'s probability density function. It is strictly proper for distributions with finite density. CRPS is defined as
\begin{equation}
    \text{CRPS}(P,y) = \int_{-\infty}^\infty \left[F_P(z) - \mathbbm{1}(y \leq z)\right]^2 dz\text{,}
\end{equation}
where $F_P$ is $P$'s cumulative density function. CRPS is strictly proper for distributions with finite first moment. The emphasis on sharpness of CRPS, while maintaining calibration, is considered a major upside~\cite{gneiting_sharpness, surival_rate_prediction}.

For the multivariate domain, popular scoring rules are the multivariate extension of the Logarithmic Score (which evaluates the negative logarithm of the multivariate density function), as well as the Energy Score \cite{gneitingRaftery}:
\begin{equation}\label{eq:ES}
\text{ES}_\beta(P, y) = \mathbbm{E}_{Y \sim P}\left[\lVert Y - y \rVert_2^\beta\right] - \frac{1}{2}\mathbbm{E}_{Y, Y' \sim P} \left[\lVert Y - Y' \rVert_2^\beta\right]
\end{equation}
Here, $\lVert . \rVert_2$ denotes the Euclidean norm and $\beta \in (0,2)$. For $\beta = 1$ the Energy Score is a multivariate extension of CRPS~\cite{gneitingRaftery}.  Both rules are strictly proper for almost all multivariate distributions (the Logarithmic Score requires finite density and the Energy Score requires $\mathbbm{E}_{Y \sim P}[\lVert Y \rVert_2^\beta] < \infty$). However, as mentioned in the introduction, both the Logarithmic and Energy Scores have known drawbacks, which demands the introduction of new strictly proper scoring rules.

\subsection{Conditional CRPS}\label{sec:CCRPS}
We propose a family of (strictly) proper scoring rules, called \emph{Conditional CRPS} (CCRPS). To introduce this scoring rule, we consider a simple example of a bivariate Gaussian distribution
\begin{align*}(Y_1, Y_2) \sim \mathcal{N}(\boldsymbol\mu,\Sigma), \text{ where } \Sigma=\begin{pmatrix}\sigma_1^2 & \rho\sigma_1\sigma_2\\ \rho\sigma_1\sigma_2& \sigma_2^2\end{pmatrix},\  \sigma_1,\sigma_2>0 \text{, and }\rho\in (-1, 1).
\end{align*}
Rather than evaluating $P(Y_1, Y_2)$ directly against an observation, we instead evaluate the first marginal distribution $P(Y_1) = \mathcal{N}(\mu_1, \sigma_1^2)$, and second conditional distribution $P(Y_2\mid Y_1 = y) = \mathcal{N}(\mu_2 + \frac{\sigma_2}{\sigma_1}\rho(y - \mu_1), (1-\rho)^2\sigma_2^2)$, against their respective univariate observations, via use univariate scoring rules. Summation over these terms then defines a new multivariate scoring rule $R$:
\begin{equation}\label{eq:CCRPS_example}
    R(P, \mathbf{y}) = \text{CRPS}(P(Y_1), y_1) + \text{CRPS}(P(Y_2\mid Y_1=y_1), y_2)\text{.}
\end{equation}
Conditional CRPS generalizes the intuition that multivariate scoring rules can be constructed by evaluating univariate conditional and marginal distributions.
\begin{definition}[Conditional CRPS]\label{def:CCRPS}
Let $P(Y)$ be a $d$-variate probability distribution over a random variable $Y = (Y_1, \ldots, Y_d)$, and let $\mathbf{y} \in \mathbbm{R}^d$. Let $\mathcal{T}=\{(v_i, \mathcal{C}_i)\}_{i=1}^q$ be a set of tuples, where $v_i \in \{1, ..., d\}$ and $\mathcal{C}_i \subseteq \{1, ..., d\} \setminus \{v_i\}$. Conditional CRPS (CCRPS) is then defined as:
\begin{align}
\mathrm{CCRPS}_{\mathcal{T}}(P(Y),\mathbf{y})
    &= \sum_{i=1}^q \mathrm{CRPS}(P(Y_{v_i} \mid Y_j = y_j\text{ for } j \in \mathcal{C}_i), y_{v_i})\text{,}
\end{align}
where $P(Y_{v_i} \mid Y_j = y_j\text{ for } j \in \mathcal{C}_i)$ denotes the conditional distribution of $Y_{v_i}$ given observations $Y_j = y_j$ for all $j \in \mathcal{C}_i$.

In the case that $P(Y_{v_i} \mid Y_j = y_j\text{ for } j \in \mathcal{C}_i)$ is ill-defined for observation $y$ (i.e. the conditioned event $Y_j = y_j\text{ for } j \in \mathcal{C}_i$ has zero likelihood or probability), we define $\text{CRPS}(P(Y_{v_i} \mid Y_j = y_j\text{ for } j \in \mathcal{C}_i), y_{v_i}) = \infty$.
\end{definition}
Conditional CRPS defines a family of scoring rules via a conditional specification $\mathcal{T}$ (see Figure \ref{fig:CCRPS}). For example, choosing $d = 2$ and $\mathcal{T} = \{(1, \emptyset), (2, \{1\})\}$ yields the rule $R$ that is defined in Equation \eqref{eq:CCRPS_example}. Conditional CRPS often defines useful scoring rules, as members are always proper, and often strictly proper:
\begin{figure}[t]
\centering
  \includegraphics[width=0.6\textwidth]{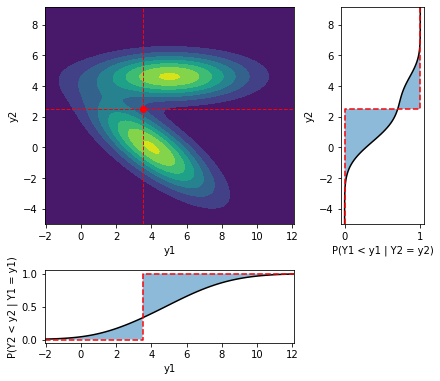}
  \caption{Visualization of Conditional CRPS, using $d = 2$ and $\mathcal{T} = \{(2, \{1\}), (1, \{2\})\}$. CCRPS evaluates an observed multivariate distribution sample by computing the distribution's univariate conditionals, conditioned on observations for other variates. \label{fig:CCRPS}}
\end{figure}

\begin{theorem}[Propriety of Conditional CRPS]\label{thrm:prop_CRPS}
    Consider CCRPS, as defined in Definition \ref{def:CCRPS}. For every choice of $\mathcal{T} = \{(v_i, \mathcal{C}_i)\}_{i=1}^q$, $\mathrm{CCRPS}_{\mathcal{T}}$ is proper for $d$-variate distributions with finite first moment.
\end{theorem}
Theorem \ref{thrm:prop_CRPS} can be easily deduced from the univariate strict propriety of CRPS, by writing the expected CCRPS score as a sum of expected CRPS scores. A formal proof is given in Appendix A.1. However, when setting some restrictions on the choice for $\mathcal{T}$, we can also prove a broad notion of strict propriety:
\begin{theorem}[Strict propriety of Conditional CRPS]
Consider CCRPS, as defined in Definition \ref{def:CCRPS}. Let $\mathcal{T}=\{(v_i, \mathcal{C}_i)\}_{i=1}^q$ be chosen such that there exists a permutation $\phi_1, \ldots, \phi_d$ of $1, \ldots, d$ such that:
\begin{equation}\label{eq:strict_proper_spec}
(\phi_j, \{\phi_1, \ldots, \phi_{j-1}\}) \in \mathcal{T} \text{ for } 1\leq j\leq d\text{.}
\end{equation}
$CCRPS_{\mathcal{T}}$ is strictly proper for all $d$-variate distributions with finite first moment, that are either discrete\footnote{I.e. distributions $P$ for which a countable set $\Omega \subset \mathbbm{R}^d$ exists such that $\mathbbm{P}_{Y \sim P}(Y \in \Omega) = 1$.}or absolutely continuous\footnote{I.e. distributions $P$ for which a Lebesgue integratable function $f_P: \mathbbm{R}^d \to [0, \infty)$ exists, such that for all measurable sets $U \subseteq \mathbbm{R}^d$, we have $\mathbbm{P}_{Y \sim P}(Y \in U) = \int_U f_P(u)du$.}.
\end{theorem}
This can be proven by using the conditional chain rule to show that any two distinct multivariate distributions differ in at least one specified conditional. Strict propriety of CRPS is then used to show strict inequality in expectancy of this CRPS term. Formal proofs are given in Appendices A.2 and A.3.

Unfortunately, there exists no CCRPS variant that is strictly proper for all distributions with finite first moment, as problems arise with distributions that are neither continuous nor discrete. This is shown in Appendix A.4.

\subsubsection{Closed-form expressions}
Unlike the Energy Score, it is surprisingly easy to find closed-form expressions for Conditional CRPS. Many popular families of multivariate distributions have marginals and conditionals which themselves are members of popular univariate distributions, many of which already have known closed-form CRPS expressions \cite{JSSv090i12}. To illustrate this, in Appendix B, in which we have provided closed-form expressions for (mixtures of) multivariate Gaussian distributions, the Dirichlet distribution, the multivariate Log-normal distribution and the multivariate student-t distribution.


\begin{figure}[t]
\centering
  \input{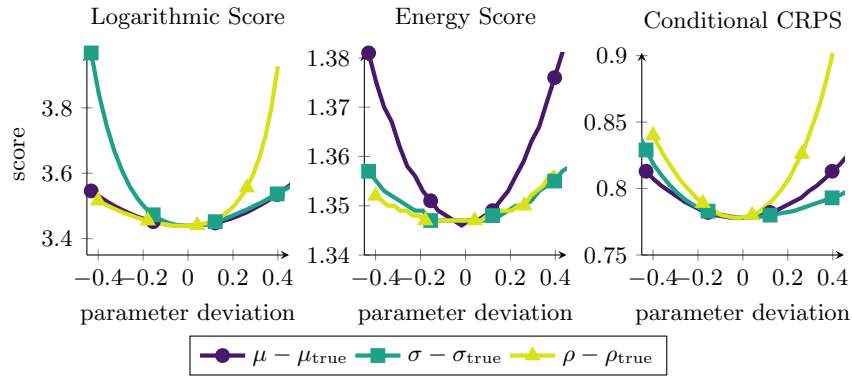} 

\pgfplotsset{
muStyle/.style={col1,mark options ={col1},mark repeat={8}, ultra thick, error bars/.cd,y dir = both, y explicit},
sigmaStyle/.style={col2,mark options ={col2},mark repeat={8}, ultra thick, error bars/.cd,y dir = both, y explicit},
rhoStyle/.style={col3,mark=triangle*,mark options ={col3},mark repeat={8},ultra thick, error bars/.cd,y dir = both, y explicit},
}

\begin{tikzpicture}
    \begin{groupplot}[group style={group size= 3 by 1, vertical sep=0.6cm},
    	height=.35\textwidth,
    	width=.35\textwidth,
        legend columns =-1, 
        axis lines = left]
        \nextgroupplot[ylabel={score}, xlabel={parameter deviation},
        	title={Logarithmic Score},
        	xmin=-0.45,xmax=0.45, ymin=3.35,
        	legend to name=zelda]
        	\addplot+[muStyle]  table[x=x,y=y] {mu_log.dat};
            \addlegendentry{$\mu-\mu_\text{true}$};
        	\addplot+[sigmaStyle]  table[x=x,y=y] {sig_log.dat};
            \addlegendentry{$\sigma-\sigma_\text{true}$};
            \addplot+[rhoStyle]  table[x=x,y=y] {rho_log.dat};
            \addlegendentry{$\rho-\rho_\text{true}$};
        \nextgroupplot[
        	title={Energy Score}, xlabel={parameter deviation},
        	xmin=-0.45,xmax=0.45, ymin=1.34,
            ]
        	\addplot+[muStyle]  table[x=x,y=y] {mu_es.dat};
        	\addplot+[sigmaStyle]  table[x=x,y=y] {sig_es.dat};
            \addplot+[rhoStyle]  table[x=x,y=y] {rho_es.dat};
        \nextgroupplot[
        	title={Conditional CRPS}, xlabel={parameter deviation},
            xmin=-0.45,xmax=0.45, ymin=0.75,
            ]
        	\addplot+[muStyle]  table[x=x,y=y] {mu_ccrps.dat};
        	\addplot+[sigmaStyle]  table[x=x,y=y] {sig_ccrps.dat};
            \addplot+[rhoStyle]  table[x=x,y=y] {rho_ccrps.dat};
    \end{groupplot} 
\end{tikzpicture}\\

\pgfplotslegendfromname{zelda}
  \caption{Plot of mean score values against the deviation of a predicted distribution parameter from the true distribution parameter. We evaluate three strictly proper scoring rules with respect to the deviation of the predicted mean, standard deviation or correlation coefficient from the data distribution ($\mu_\text{true} = 1$, $\sigma_\text{true} = 1$ and $\rho_\text{true} = 0.4$). See Appendix D.  \label{fig:score_comparison}}
\end{figure}

\subsubsection{Correlation sensitivity}
Conditional CRPS displays promising advantages over the Energy and the Logarithmic Score with regard to correlation sensitivity. We evaluate the correlation sensitivity by a small experiment, similar to the one by Pinson and Tastu \cite{pinsontastu}. Here, we investigate the increase in expected scores when the forecasted distribution deviates from the data distribution in either the mean, standard deviation, or correlation coefficient. The data generating algorithm is described in Appendix D. We compare three scoring rules: the Logarithmic, the Energy Score, and CCRPS with $\mathcal{T} = ((1,\{2\}),(2,\{1\}))$. Figure \ref{fig:score_comparison} shows that the CCRPS score increases more with the prediction error in $\rho$ than the Logarithmic and the Energy score. Therewith, the CCRPS score fixes the well documented lack of correlation sensitivity of the Energy Score \cite{alexander, pinsontastu}.

\subsection{CCRPS as ANN Loss Function for Multivariate Gaussian Mixtures}
We show an application of Conditional CRPS as a loss function that allows for the numerical optimization of Artificial Neural Networks (ANNs)~\cite{NN_explanation} to return the parameters of the predicted distribution of target variables in a regression task. 
We assume that the target distribution is a mixture of $m$ $d$-variate Gaussian distributions. This distribution is defined by $m$ mean vectors $\boldsymbol{\mu}_1, \ldots, \boldsymbol{\mu}_m \in \mathbbm{R}^d$, $m$ positive-definite matrices $\Sigma_1, \ldots, \Sigma_m \in \mathbbm{R}^{d \times d}$, and $m$ weights $w_1, \ldots, w_m \in [0,1]$ such that $\sum_{i=1}^m w_i = 1$. A multivariate mixture Gaussian $P$ defined by these parameters is then given by the density function
\begin{equation}\label{eq:MixGauss1}
    f_P(\mathbf{y}) = \sum_{l=1}^m w_l \cdot f_{\mathcal{N}(\boldsymbol{\mu}_l, \Sigma_l)}(\mathbf{y}) = \sum_{l=1}^m w_l \frac{\exp\left(-\frac{1}{2}( \mathbf{y}-\boldsymbol{\mu}_l)^\top \Sigma_i^{-1}(\mathbf{y}-\boldsymbol{\mu}_l)\right)}{\sqrt{(2\pi)^d \cdot \left|\Sigma_l\right|}}\text{.}
\end{equation}
That is, the ANN returns for each input $\mathbf{x}$ a set of parameters $\{(\boldsymbol{\mu_l},w_l, L_l)\}_{l=1}^m$, where
$L_l\in \mathbbm{R}^{d \times d}$ is a Cholesky lower matrix~\cite{MUSCHINSKI2022}, defining a positive-definite matrix
$\Sigma_i = L_i \cdot L_i^\top$\text{.}
Given a dataset $(\mathbf{x}_i, \mathbf{y}_i)_{i=1}^n$, and an ANN $\theta(\mathbf{x})$ that predicts the parameters of a $d$-variate mixture Gaussian distribution, we can define a loss function over the mean CCRPS score:
\begin{equation}
    \mathcal{L}(\theta, (\mathbf{x}_i, \mathbf{y}_i)_{i=1}^n) = \frac{1}{n}\sum_{i=1}^n \mathrm{CCRPS}_{\mathcal{T}}(P_{\theta(\mathbf{x}_i)}, \mathbf{y}_i)\text{.}
\end{equation}
Unfortunately, if we choose $\mathcal{T}$ such that the loss function computes mixture Gaussian distributions conditioned on $c$ variables, then we require matrix inversions of $c \times c$ matrices (cf.\@ Appendix B).\footnote{Support for backpropagation through matrix inversions is offered in packages such as Tensorflow. However, for larger matrices, gradients can become increasingly unstable.} Therefore, we choose a simple Conditional CRPS variant that conditions on at most one variable, using $\mathcal{T}_0 = \{(i, \emptyset)\}_{i=1}^d \cup \{(i, \{j\})\}_{i \neq j}^d$. That is,
\begin{align*}
\mathrm{CCRPS}_{\mathcal{T}_0}(P, \mathbf{y}) &= \sum_{i=1}^d\mathrm{CRPS}(P(Y_i), y_i) + \sum_{j \neq i} \mathrm{CRPS}(P(Y_i | Y_j = y_j), y_i)
.
\end{align*}
Using this definition, we find an expression for this variant of CCRPS. As both $P(Y_i|Y_j = y_j)$ and $P(Y_i)$ are univariate mixture Gaussian distributions, computing $\mathrm{CCRPS}_{\mathcal{T}_0}(P, y)$ is done by simply computing the parameters of these distributions, and applying them in a CRPS expression for univariate mixture Gaussian distributions given by Grimit et al. \cite{grimit}:\begin{theorem}[CCRPS expression for multivariate mixture Gaussians]
    Let $P$ be a mixture of $m$ $d$-variate Gaussians, as defined in Equation \eqref{eq:MixGauss1} via $\boldsymbol{\mu}_k \in \mathbbm{R}^d$, $\Sigma_k \in \mathbbm{R}^{d \times d}$ and $w_k \in [0,1]^m$ for $1\leq k\leq m$. Then we have for $\mathbf{y} \in \mathbbm{R}^d$:
    
    \begin{align*}
    \begin{split}
        &\mathrm{CCRPS}_{\mathcal{T}_0}(P, \mathbf{y})=\\ &\sum_{1\leq i \neq j \leq d} \left[\sum_{k=1}^m \hat w_{kj} H(y_i - \hat \mu_{kij}, \hat \Sigma_{kij})
    - \frac{1}{2} \sum_{k,l=1}^m \hat w_{kj} \hat w_{lj} H(\hat \mu_{kij} - \hat \mu_{lij}, \hat \Sigma_{kij} + \hat \Sigma_{lij})\right]\\
    &+ \sum_{i = 1}^d \left[\sum_{k=1}^m w_k H(y_i - \mu_{k,i}, \Sigma_{k,ii})
    - \frac{1}{2} \sum_{k,l=1}^m w_k w_l H(\mu_{k,i} - \mu_{l,i}, \Sigma_{k,ii} + \Sigma_{k,ll})\right]  \end{split}
    \end{align*}
    Here:
    \begin{itemize}
        \item $H(\mu, \sigma^2) = \mu \left(2\Phi\left(\frac{\mu}{\sigma}\right)-1\right) + 2\sigma\varphi\left(\frac{\mu}{\sigma}\right)$, where $\varphi$ and $\Phi$ denote the PDF and CDF of the standard Gaussian distribution,
        \item $\displaystyle\hat w_{kj} = \frac{w_k \cdot f_{\mathcal{N}(\mu_{k,j},\Sigma_{k,jj})}(y_j)}{\sum_{l=1}^m w_l \cdot f_{\mathcal{N}(\mu_{l,j},\Sigma_{l,jj})}(y_j)}$, 
        \item $\hat \mu_{kij} = \mu_{k,j} + \frac{\Sigma_{k,ij}}{\Sigma_{k,jj}}\left(y_j - \mu_{k,j}\right)$,
        \item $
            \hat \Sigma_{kij} = \Sigma_{k,ii}-\frac{(\Sigma_{k,ij})^2}{\Sigma_{k,jj}}$.
        \end{itemize}
    \end{theorem}
In Appendix B, we state an expression for the more generic case $\text{CRPS}(P(Y_i|Y_j = y_j \text{ for } j \in \mathcal{C}_j), y_i)$.
\begin{figure*}[t]
  \includegraphics[width=\textwidth]{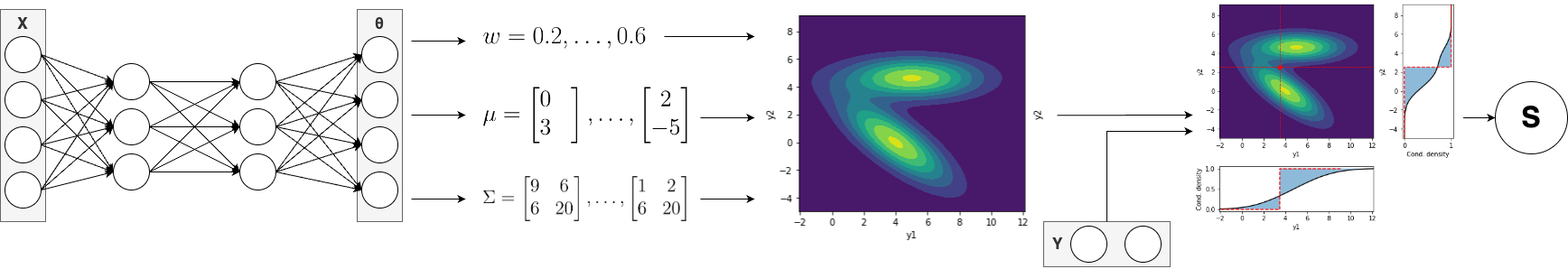}
  \caption{The output layer of a CCRPS network defines a set of weights, mean vectors and positive definite matrices, via a combination of activation functions and Cholesky parameterizations. These parameters define the predicted multivariate mixture Gaussian distribution, which is evaluated against an observation via CCRPS loss.\label{fig:CCRPS_overivew}}
\end{figure*}
An overview of the proposed mixture Gaussian CCRPS ANN approach is given in Figure \ref{fig:CCRPS_overivew}. The approach to predict a mixture model via a single network contrasts the multiple-network approach via bagging, used by a.o. Carney et al. \cite{Carney_2005}, and simplifies the architecture.

\subsection{Energy Score Ensemble Models}
Secondly, we propose an ANN loss variant that empirically approximates the Energy Score. The energy score (cf.\@  Equation \eqref{eq:ES}) is defined over expected values, for which no closed-form expression exists, that would enable the computation of a gradient. However, the Energy Score is fairly easily approximated by an ensemble of randomly sampled vectors. Let $P$ be a $d$-variate distribution, and let $\hat{\mathbf{y}}_1, \ldots, \hat{\mathbf{y}}_m$ be independent samples of the distribution $P$. We approximate $P$ by its empirical distribution function by assigning probability $\frac{1}{m}$ to each sampled vector $\hat{\mathbf{y}}_l$. That is, we use the \emph{stepwise} multivariate CDF approximation:
\begin{equation}
F_P(\mathbf{z}) \approx \frac{1}{m}\sum_{l=1}^m \prod_{i=1}^d \mathbbm{1}(\hat{y}_{l,i} \leq z_i)
\end{equation}
We can now approximate the Energy Score:
\begin{equation}\label{eq:ES_approx}
    \text{ES}(P, \mathbf{y}) \approx \frac{1}{m}\sum_{l=1}^m \lVert \hat{\mathbf{y}}_l - \mathbf{y} \rVert_\varepsilon - \frac{1}{2m^2} \sum_{k,l=1}^m \lVert \hat{\mathbf{y}}_k - \hat{\mathbf{y}}_l\rVert_\varepsilon\text{.}
\end{equation}
Here, $\lVert.\rVert_\varepsilon$ is the $\varepsilon$-smoothed Euclidean norm
$
\lVert \mathbf{v} \rVert_\varepsilon = \sqrt{\varepsilon + \lVert \mathbf{v}\rVert^2_2}
$,
for some small $\varepsilon > 0$. The $\varepsilon$-smoothed Euclidean norm makes the norm differentiable, even at $\mathbf{v}=0$. This approximation allows for numerical optimization, in which a model predicts $P$ indirectly over $\hat{\mathbf{y}}_1, \ldots, \hat{\mathbf{y}}_m$. That is, we can train an ANN to return for each feature vector $\mathbf{x}$ the distribution defining parameters $\theta(\mathbf{x})=\hat{\mathbf{y}}_1, \ldots, \hat{\mathbf{y}}_m$, using the loss defined in Equation \eqref{eq:ES_approx}. This approach is similar to the recent, independently developed work by Kanazawa and Gupta~\cite{es_model}, and can be considered a non-generative and conditioned version of their approach.
\section{Experiments}
We compare the probabilistic predicted performance of the newly proposed methods to state-of-the-art probabilistic regression methods. We provide our source code online.\footnote{\url{https://github.com/DaanR/scoringrule_networks}} As competitors, we choose the best-performing models of the comparative study from Ćevid et al. \cite{DRF}, and the Logarithmic Score trained networks.
\begin{itemize}
    \item \textbf{Distributional Random Forest (DRF)} \cite{DRF} is a random forest regression model with an adapted splitting criterion for target vectors (based on MMD approximations), and an adapted aggregation that returns a weighted ensemble of target vectors.
    \item \textbf{Conditional GAN (CGAN)} \cite{C-GAN-regression} is an extension of the popular Generative Adverserial Network. Except, the model is ``conditioned'' on input $x$ by adding it as input to both generator and discriminator.
    \item \textbf{Distributional k-nearest neighbors (kNN)}~\cite{DRF} predicts a distribution in which each of the k-nearest neighbors is assigned $\frac{1}{k}$ probability.
    \item \textbf{Mixture MLE neural networks} (a.o. ~\cite{mle}) are the closest to our approach. MLE ANNs use the Logarithmic Score as loss function. We employ the same architectures as MLE networks in our CCRPS networks.
\end{itemize}
For DRF and CGAN, we use implementations provided by the authors. For mixture MLE networks and kNN, we used our own implementations. Similar to CCRPS, we prevent the numerically unstable backpropagation through large matrix inverses by applying the Logarithmic Score on all bivariate marginal densities $P(Y_i, Y_j)$, rather than on the multivariate density. This way, we could improve on the originally proposed implementation of MLE minimization \cite{mle}:
\begin{equation}
    \mathrm{MLE}_\text{biv}(P(Y),\mathbf{y}) = -\sum_{1\leq i \neq j\leq d} \log f_{P(Y_i, Y_j)}(y_i, y_j).
\end{equation}
$\mathrm{MLE}_\text{biv}$ is strictly proper for $d \leq 2$ and proper for $d > 2$. For both, MLE-trained networks and CCRPS-trained networks, we try variants with $m \in \{1,10\}$ (Gaussian) mixture distributions. For each model and each experiment, we choose the best hyperparameters and architecture out of a variety of hyperparameters/architectures, based on the validation Energy Score. Furthermore, for all ANN-based models, we use the validation set loss as a training cutoff criterion: training is stopped once the validation set increases compared to the previous epoch.
\subsection{Evaluation metrics}
Unfortunately, there is no clear consensus on appropriate evaluation metrics for multivariate distributional regression models~\cite{alexander}. Hence, we choose a variety of popular metrics: the Energy Score (cf.\@ Equation \eqref{eq:ES}) with $\beta = 1$, and the Variogram Score \cite{variogram_score} with $\beta \in \{0.5, 1,2\}$:
\begin{equation}
    \text{VS}_{\beta}(P, \mathbf{y}) = \sum_{1\leq i < j\leq d} \left(|y_i - y_j|^\beta - \mathbbm{E}_{Y \sim P}\left[|Y_i - Y_j|^\beta\right]\right)^2\text{.}
\end{equation}
The Variogram Score is only proper but usually better at evaluating errors in the forecasted correlation than the Energy Score \cite{alexander}. For most models, the scores are approximated via Monte Carlo approximations (see Appendix C for details).

Contrary to the comparative studies done by Aggarwal et al. \cite{C-GAN-regression} and {\'C}evid et al. \cite{DRF}, we decide not to use the Logarithmic Score (also named NLPD) as evaluation metric, since the ES ensemble model, kNN, C-GAN and DRF do not predict an explicit density function, and we found that the Logarithmic Score is fairly dependent on the choice of density estimation for the post-processing. All datasets are split into training, validation, and testing dataset. We summarize dataset statistics in Table~\ref{tab:dataset_summary}.

\begin{table}[t]
    \centering
    \caption{Dataset statistics: input dimensionality ($p$), target dimensionality ($d$), as well as training ($n_\text{train}$), validation ($n_\text{val}$) and testing ($n_\text{test}$) dataset sizes. For the synthetic datasets, the morphing function is also listed.}
    \begin{minipage}[t]{66mm}
    \begin{tabular}{l@{\hskip 0.1in}lrrrrr}
        \toprule
        \bfseries Name & \texttt{morph}(y') & $p$ & $d$ &  $n_\text{train}$ & $n_\text{val}$ & $n_\text{test}$ \\
        \midrule
        \bfseries Gauss 2D & $2 y' + 2$ & 40 & 2 & 6K & 2K & 2K\\
        \bfseries Gauss 5D & $2 y' + 2$ & 100 & 5 & 6K & 2K & 2K\\
        \bfseries Quadratic & $y'^2$ & 40 & 2 & 6K & 2K & 2K\\
        \bottomrule
    \end{tabular}
    \end{minipage}
    \begin{minipage}[t]{50mm}
    \begin{tabular}{l@{\hskip 0.1in}rrrrr}
        \toprule
        \bfseries Name & $p$ & $d$ &  $n_\text{train}$ & $n_\text{val}$ & $n_\text{test}$ \\
        \midrule
        \bfseries Births & 23 & 2 & 18K & 6K & 6K\\
        \bfseries Air & 25 & 6 & 26K & 8.5K & 8.5K\\
        \bfseries GR--GEM & 8 & 8 & 18K & 6K & 6K\\
        \bfseries GR--GFS &  8 & 8 & 18K & 6K & 6K\\
        \bfseries GR--GFS & 160 & 8 & 18K & 6K & 6K\\
        \bfseries GR--comb. & 176 & 8 & 18K & 6K & 6K\\
        \bfseries GR \& DR & 176 & 48 & 14K & 4.5K & 4.5K\\
        \bottomrule
    \end{tabular}
    \end{minipage}
    \label{tab:dataset_summary}
\end{table}


\subsection{Synthetic Experiments}
We base our data generation process for the synthetic experiments on the task to post-process an ensemble model. This model is for example applied in the task of weather forecasts (cf. experiments on the global radiation data in Section~\ref{sec:rw_exp}). Here, a distributional regression model receives $s$ (non probabilistic) predictions $\mathbf{v}_1, \ldots, \mathbf{v}_s \in \mathbbm{R}^d$ for a target variable $\mathbf{y} \in \mathbbm{R}^d$. That is, the probabilistic regression model is supposed to learn the target distribution from the distribution of target predictions of an ensemble of models $\mathbf{v}_1, \ldots, \mathbf{v}_s \in \mathbbm{R}^d$. In other words, the probabilistic regression model is trained to correct ensemble predictions. 
For each observation, we sample $s=20$ i.i.d. vectors from a Gaussian with randomly chosen parameters, and sample the target vector from the same distribution. To further simulate errors in the ensemble predictions, we apply a morphing operation (either $\texttt{morph}(y') = 2y' + 2$ or $\texttt{morph}(y') = y'^2$) on the target vector. An overview of morphing functions is given in Table \ref{tab:dataset_summary}.
\begin{algorithm}
    \caption{Synthetic data sampling of a single $(\mathbf{x}, \mathbf{y})$ pair.}\label{alg:pred_observ_sampling}
    \begin{algorithmic}
        \Function{GenerateRegressionData}{$\texttt{morph}, d, s = 20$}
            \State Sample $\boldsymbol\mu\in\mathbbm{R}^d$ such that $\mu_j\sim \mathcal{N}_1(0,1)$\hfill\Comment{Choose a random mean vector}
            \State  $L\leftarrow 0 \in\mathbbm{R}^{d\times d}$
            \State Sample $L_{jl}\sim \mathcal{N}_1(0,1)$ for $j\geq l$\hfill\Comment{Choose a random Cholensky lower matrix}
            \State $L_{jj}\leftarrow \lvert L_{jj}\rvert$ for $1\leq j\leq d$ \hfill \Comment{Ensure a strictly positive diagonal}
            \For{ $r\in\{1,\ldots,s\}$}
             \State Sample $\mathbf{v}_r\sim \mathcal{N}(\boldsymbol\mu, LL^\top)$ \hfill\Comment{Each $\mathbf{v}_r$ is a $d$-dimensional vector}
            \EndFor
            \State $\mathbf{x} \leftarrow \texttt{flatten}(\mathbf{v}_1, \ldots, \mathbf{v}_{s})$\hfill\Comment{The input is a vector of length $d \cdot s$}
            \State Sample $\mathbf{y}' \sim \mathcal{N}(\boldsymbol\mu, LL^\top)$. \hfill\Comment{Sample a $d$-dimensional vector i.i.d. to $\mathbf{v}_1, \ldots, \mathbf{v}_{s}$}
            \State $\mathbf y_i \leftarrow \texttt{morph}(\mathbf{y}'_i)$ for $1 \leq i \leq d$. \hfill\Comment{Apply a simple morph to the target vector}
            \State\Return $(\mathbf{x, y})$
        \EndFunction
    \end{algorithmic}
\end{algorithm}

The experiment results have been summarized in Table \ref{tab:synth_results}. We note that ANN models seem particularly suited for the chosen experiments, with the CCRPS mixture model outperforming the other models on 6 of the 12 evaluated metrics.

\begin{table}
\caption{Synthetic experiment evaluation metrics (ES, VS) are displayed in a group of four rows, and the best score is highlighted. * Scores divided by $10^7$. }\label{tab:synth_results}
\begin{tabular}{l@{\hskip 0.1in}l@{\hskip 0.1in}llll@{\hskip 0.15in}l@{\hskip 0.1in}l@{\hskip 0.1in}l@{\hskip 0.1in}l}
\toprule
& & \begin{tabular}[c]{@{}l@{}}\bfseries CCRPS\\ Gauss.\end{tabular} & \begin{tabular}[c]{@{}l@{}}\bfseries CCRPS\\ mixt.\end{tabular} & \begin{tabular}[c]{@{}l@{}}\bfseries ES Ens.\\ 100 pts.\end{tabular} & \begin{tabular}[c]{@{}l@{}}\bfseries MLE\\ Gauss.\end{tabular} & \begin{tabular}[c]{@{}l@{}}\bfseries MLE\\ mixt.\end{tabular} & \begin{tabular}[c]{@{}l@{}} \bfseries KNN\\ \\\end{tabular} & \begin{tabular}[c]{@{}l@{}} \bfseries CGAN\\ \\\end{tabular} & \begin{tabular}[c]{@{}l@{}} \bfseries DRF\\ \\\end{tabular}\\

 \midrule
\multirow{4}{*}{\rotatebox[origin=c]{90}{2D-Gauss.}} & ES & 2.1426 & 2.1427 & 2.1233 & 2.1261 & \bfseries{2.0953} & 2.1918 & 2.1983 & 2.1753\\
& VS$_{0.5}$ & 0.5257 & \bfseries{0.5018} & 0.5112 & 0.5196 & 0.5042 & 0.5734 & 0.5199 & 0.5364\\
 & VS$_1$ & 8.0213 & \bfseries{7.6323} & 7.8361 & 7.9425 & 7.6766 & 8.7665 & 7.8844 & 8.1639\\
 & VS$_2$ & 1553.6 & \bfseries{1494.0} & 1534.8 & 1560.8 & 1519.8 & 1656.8 & 1517.5 & 1558.3\\
 \midrule
 \multirow{4}{*}{\rotatebox[origin=c]{90}{5D-Gauss.}} & ES & 5.1889 & \bfseries 5.1691 & 5.1925 & 5.2693 & 5.1903 & 5.2804 & 5.4533 & 5.4577 \\
& VS$_{0.5}$ & 6.7784 & 6.7728 & 6.7921 & 6.9214 & \bfseries 6.7612 & 7.2873 & 7.1653 & 7.0541\\
 & VS$_1$ & 120.67 & 120.56 & 120.53 & 123.71 & \bfseries 120.35 & 131.24 & 129.04 & 127.17\\
 & VS$_2$ & 30087 & 30149 & \bfseries 29996 & 30935 & 30073 & 31233 & 30756 & 30473\\
 \midrule
  \multirow{4}{*}{\rotatebox[origin=c]{90}{Quadratic}} & ES & 2.7206 & \bfseries 2.6668 & 2.6764 & 2.7989 & 2.6678 & 2.8597 & 2.7766 & 2.6745 \\
& VS$_{0.5}$ & \bfseries{1.0087} & 1.0297 & 1.0228 & 1.0926 & 1.0216 & 1.1951 & 1.0245 & 1.0638 \\
 & VS$_1$ & 33.371 & 33.851& \bfseries{33.741} & 34.796 &  34.055 & 37.589 & 34.912 & 35.386 \\
 & VS$_2$* & 130.44 & \bfseries{128.34} & 129.78 & 142.05 & 130.53 & 133.27 & 133.16 & 132.19\\
\bottomrule
\\
\end{tabular}
\label{tab:exp_results}
\end{table}

\subsection{Real World Experiments}\label{sec:rw_exp}
We evaluate our method on a series of real-world datasets for multivariate regression. All datasets are normalized for each input and target field based on the training dataset mean and standard deviation.
\begin{enumerate}
    \item \textbf{Births dataset}~\cite{DRF}: prediction of pregnancy duration (in weeks) and a newborn baby's birthweight (in grams) based on statistics of both parents.
    \item \textbf{Air quality dataset}~\cite{DRF}: Predictio of the concentration of six pollutants (NO$_2$, SO$_2$, CO, O$_3$, PM$_{2.5}$ and PM$_{10}$) based on statistics about the measurement conditions (e.g., place and time)
    \item \textbf{Global radiation dataset}: Prediction of solar radiation based on three numerical weather prediction (NWP) models (the single-model run models GEM~\cite{GEM} and GFS~\cite{GFS} and the 20-ensemble model run GEPS~\cite{GEPS}), as well as global radiation (GR) measurements at weather stations in the Netherlands~\cite{KNMI_uurdata} and Germany~\cite{DWD_uurdata}. Models receive an NWP forecast as input, and a station measurement as target. 
    In our experiments, models predict an 8-variate distribution, consisting of three-hour GR averages. We run four different experiments, in which models receive either GEM, GFS, GEPS or all three NWP sources as input.
    
    \item \textbf{Global-diffuse radiation dataset}: Prediction of 24 hourly global and diffuse radiation (DR) station measurements based on all three NWP sources (like in the global radiation dataset). 

\end{enumerate}

\begin{table}[!t]
\centering
\caption{Real-world experiment metrics (ES, VS) are displayed in a group of four rows, and the best score is highlighted.\label{tab:rw_results}} 
\begin{tabular}{l@{\hskip 0.1in}l@{\hskip 0.1in}llll@{\hskip 0.15in}l@{\hskip 0.1in}l@{\hskip 0.1in}l@{\hskip 0.1in}l}
\toprule
& & \begin{tabular}[c]{@{}l@{}}\bfseries CCRPS\\ Gauss.\end{tabular} & \begin{tabular}[c]{@{}l@{}}\bfseries CCRPS\\ mixt.\end{tabular} & \begin{tabular}[c]{@{}l@{}}\bfseries ES Ens.\\ 100 pts.\end{tabular} & \begin{tabular}[c]{@{}l@{}}\bfseries MLE\\ Gauss.\end{tabular} & \begin{tabular}[c]{@{}l@{}}\bfseries MLE\\ mixt.\end{tabular} & \begin{tabular}[c]{@{}l@{}} \bfseries KNN\\ \\\end{tabular} & \begin{tabular}[c]{@{}l@{}} \bfseries CGAN\\ \\\end{tabular} & \begin{tabular}[c]{@{}l@{}} \bfseries DRF\\ \\\end{tabular}\\

 \midrule
\multirow{4}{*}{\rotatebox[origin=c]{90}{Births}} & ES & 0.6969 & 0.6891 & 0.6897 & 0.7025 & \bfseries 0.6881 & 0.7028 & 0.7140 & 0.6924 \\
 & VS$_{0.5}$ & 0.1039 & \bfseries{0.1034}& 0.1035 & 0.1041 & \bfseries{1.1034} & 0.1052 & 0.1063 & 0.1039  \\
 & VS$_1$ & 0.2627 & 0.2612 & 0.2612 & 0.2632 & \bfseries{0.2608} & 0.2657 & 0.2688 & 0.2627 \\
 &  VS$_2$ & 1.3137 & 1.3069 & 1.3069 & 1.3147 & \bfseries{1.3000} & 1.3258 & 1.3353 & 1.3121 \\
 \midrule
\multirow{4}{*}{\rotatebox[origin=c]{90}{Air}} & ES & 1.0912 & 1.0844 & 1.0887 & 1.0919 & 1.0881 & 1.1420 & 1.1887 & \bfseries{1.0683} \\
 & VS$_{0.5}$ & 1.8501 &1.8380 & 1.8471 & 1.9146 & 1.8562 & 1.9490 & 2.0619 & \bfseries{1.8041}\\
 & VS$_1$ & 8.5582 & 8.5118 &  8.5433 & 8.5348 & 8.8644 & 9.0913 & 9.6012 &  \bfseries{8.3894} \\
 & VS$_2$ & 578.74 & \bfseries{572.88} &  572.90 & 584.74 & 575.10 & 591.58 & 605.32 & 574.34 \\
 \midrule
\multirow{4}{*}{\rotatebox[origin=c]{90}{GR--GEM}} & ES & 0.0988 & 0.0983 & 0.0978 & 0.1013 & 0.0983 & 0.0998 & 0.1216 & \bfseries{0.0960} \\
 & VS$_{0.5}$ & 0.2476 & 0.2544 & 0.2267 & 0.2492 & 0.2459 & 0.2046 & 0.3114 & \bfseries{0.1926} \\
 & VS$_1$ & 0.1705 & 0.1704 & 0.1692 & 0.1742 & 0.1710 & 0.1743 & 0.2273 & \bfseries{0.1657} \\
 & VS$_2$ & 0.1026 & 0.1028 & 0.1019 & 0.1053 & 0.1022 & 0.1043 &  0.1363 & \bfseries{0.1001} \\
 \midrule
\multirow{4}{*}{\rotatebox[origin=c]{90}{GR--GFS}} & ES & 0.1020 & 0.1006 & 0.0998 & 0.1014 &  0.0992 & 0.1058 & 0.1525 & \bfseries{0.0968} \\
 & VS$_{0.5}$ & 0.2590 & 0.2636 & 0.2448 & 0.2530 & 0.2519 &  0.2273 &  0.5265 & \bfseries{0.1978} \\
 & VS$_1$ & 0.1752 & 0.1744 & 0.1740 & 0.1732 & 0.1718 & 0.1899 & 0.3066 & \bfseries{0.1667} \\
 & VS$_2$ & 0.1033 & 0.1042 & 0.1025 & 0.1027 & 0.1004 & 0.1113 & 0.1723 & \bfseries{0.0993} \\
 \midrule
\multirow{4}{*}{\rotatebox[origin=c]{90}{GR--GEPS}} & ES & 0.0759 & \bfseries{0.0722} & 0.0723 & 0.0820 &  0.0755 & 0.0806 & 0.1195 & 0.0840 \\
&VS$_{0.5}$ & 0.1800 & 0.1760 & \bfseries{0.1430} & 0.1950 & 0.1793 & 0.1743 & 0.3129 & 0.1589\\
 & VS$_1$ & 0.1070 & 0.0988 & \bfseries{0.0966} & 0.1245 & 0.1120 & 0.1247 & 0.2177 & 0.1589 \\
 & VS$_2$ & 0.0614 & 0.0588 & \bfseries{0.0566} & 0.0748 & 0.0672 & 0.0733 & 0.1236 & 0.0800\\
 \midrule
\multirow{4}{*}{\rotatebox[origin=c]{90}{GR--comb.}} & ES & 0.0734 & \bfseries{0.0718} & 0.0723 & 0.0766 & 0.0745 & 0.0795 & 0.1209 & 0.0828 \\
& VS$_{0.5}$ & 0.1756 & 0.1754 & 0.1496  & 0.1806 & 0.1765 & \bfseries{0.1438} & 0.3202 & 0.1550 \\
 & VS$_1$ & 0.1017 & \bfseries{0.0985} & 0.0998 & 0.1119 & 0.1085 & 0.1218 & 0.2254 & 0.1318 \\
 & VS$_2$ & 0.0596 & 0.0595 & \bfseries{0.0594} & 0.0672 & 0.0638 &0.0722 & 0.1346 & 0.0785\\
 \midrule
\multirow{4}{*}{\rotatebox[origin=c]{90}{GR \& DR}} & ES & 1.4269 & \bfseries{1.3924}& 1.4229 & 1.4819 & 1.4873 & 1.5564 & 1.8090 & 1.5464 \\
& VS$_{0.5}$ & \bfseries 46.395 & 47.592 & 51.205 & 48.045 & 46.413 & 51.027 & 77.996 & 51.104 \\
 &VS$_1$ & 210.27 & \bfseries{204.60} &  208.41 & 224.72 & 216.43 & 242.05 & 311.35 & 245.35 \\
 & VS$_2$ & 3593.1 & \bfseries{3419.0} & 3503.3 & 3775.3 & 3715.8 & 3963.4 & 5983.5 & 4037.2\\
\bottomrule
\\
\end{tabular}
\label{tab:exp_results}
\end{table}
\begin{figure}[t]
    \centering
    \includegraphics[width=0.22\textwidth]{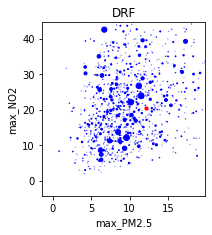}
    \hfill
    \includegraphics[width=0.22\textwidth]{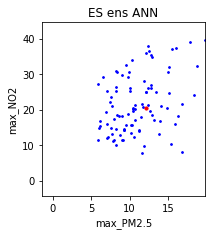}
    \hfill
    \includegraphics[width=0.22\textwidth]{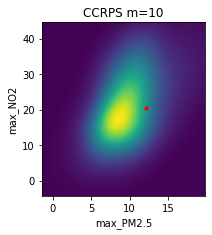}
    \hfill
    \includegraphics[width=0.22\textwidth]{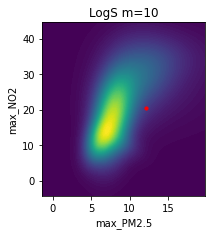}
    \caption{NO$_2$ (in ng/m$_3$) and PM$_{2.5}$ (in p.p.b.) predictions of the best four models for an entry in the ''air'' experiment testing set. The red dot denotes the target measurement.}
    \label{fig:air_quality_vis}
    \centering
    \includegraphics[width=0.45\textwidth]{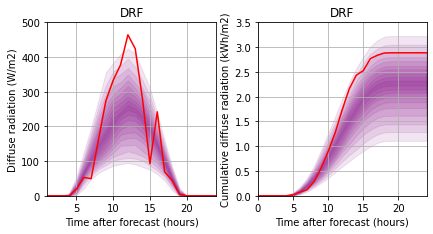}
    \hfill
    \includegraphics[width=0.45\textwidth]{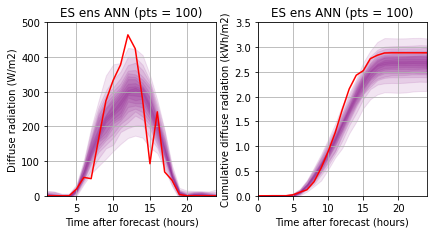}
    \includegraphics[width=0.45\textwidth]{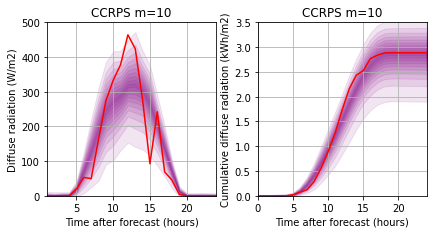}
    \hfill
    \includegraphics[width=0.45\textwidth]{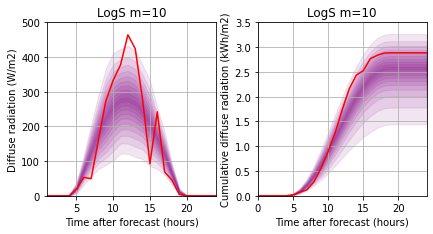}
    \caption{Diffuse irradiance predictions (95\% confidence intervals) of the best four models for an entry in the ``GR-DR'' experiment testing set. Both the marginal and cumulative distributions are visualized. The red line denotes the target measurement.}
    \label{fig:solar_rad_vis}
\end{figure}
The experiment results have been summarized in Table \ref{tab:exp_results}. Here, the testing set evaluation metrics have been listed to evaluate predictive performance on unseen data. The newly proposed models are on par with current state-of-the-art, outscoring them on about half (13 of the 28) of the evaluated metrics. CCRPS trained models do seem to outperform their MLE trained equivalents: the Gaussian CCRPS models outperform their MLE counterparts 23 out of 28 times, and the mixture CCRPS models outperform their MLE counterparts 18 out of 28 times.

However, in our experiments, none of the evaluated models consistently outperforms all other models.  Generally, one of four models (the MLE and CCRPS mixture models, ES ensemble model and DRF) scored best, with CCRPS and DRF scoring best most often. Unfortunately, we have not been able to link the relative model performances to the experiment's characteristics, as there seems to be no clear connection between the experiment nature (temporal data, tabular data or synthetic data) and the relative model results.

Finally, some visualizations of predicted distributions and their target variables have been made in Figures \ref{fig:air_quality_vis} and \ref{fig:solar_rad_vis}.



\section{Conclusion}
We propose two new loss functions for multivariate probabilistic regression models: Conditional CRPS and the approximated Energy Score. CCRPS is a novel class of (strictly) proper scoring rules, which combines some of the desirable characteristics (suitability for numerical optimization, sensitivity to correlation, and increased sharpness) from the Energy and Logarithmic Scores.

Conditional CRPS, when applied in the right setting, leads to an increase in sharpness while retaining calibration. We parameterize our regression models by means of an Artificial Neural Network (ANN), which returns for a given feature vector $\mathbf{x}$ the parameters of the predicted (conditional) target distribution. Models trained with CCRPS outperform equivalent models trained with MLE on the majority of evaluated experiments. Moreover, the novel models, trained with CCRPS and Energy Score loss, have predictive performances on par with non-parametric state-of-the-art approaches, such as DRF (cf. Tables \ref{tab:synth_results} and \ref{tab:rw_results}).

%
%
%
\bibliographystyle{splncs04}
\bibliography{main}

\appendix
\section{Proofs}
Recall Conditional CRPS, as defined in Section 2.2. Let $P(Y)$ be a $d$-variate probability distribution over a random variable $Y = (Y_1, \ldots, Y_d)$, and let $\mathbf{y} \in \mathbbm{R}^d$. Let $\mathcal{T}=\{(v_i, \mathcal{C}_i)\}_{i=1}^q$ be a set of tuples, where $v_i \in \{1, ..., d\}$ and $\mathcal{C}_i \subseteq \{1, ..., d\} \setminus \{v_i\}$. Conditional CRPS (CCRPS) is then defined as:
\begin{align}
\mathrm{CCRPS}_{\mathcal{T}}(P(Y),\mathbf{y})
    &= \sum_{i=1}^q \mathrm{CRPS}(P(Y_{v_i} \mid Y_j = y_j\text{ for } j \in \mathcal{C}_i), y_{v_i})\text{,}
\end{align}
where $P(Y_{v_i} \mid Y_j = y_j\text{ for } j \in \mathcal{C}_i)$ denotes the conditional distribution of $Y_{v_i}$ given observations $Y_j = y_j$ for all $j \in \mathcal{C}_i$. In the case that $P(Y_{v_i} \mid Y_j = y_j\text{ for } j \in \mathcal{C}_i)$ is ill-defined for observation $y$ (i.e. the conditioned event $Y_j = y_j\text{ for } j \in \mathcal{C}_i$ has zero likelihood or probability), we define $\text{CRPS}(P(Y_{v_i} \mid Y_j = y_j\text{ for } j \in \mathcal{C}_i), y_{v_i}) = \infty$.

In this appendix, we will provide formal proofs for the various (strict) propriety claims made in the main paper. First we provide some helper proofs. We first show that if the true distribution has finite first moment, Conditional CRPS is finite for a correct prediction. This is used during the (strict) propriety proofs.
\begin{lemma}[Finiteness of Conditional CRPS for correct predictions]\label{applem:finiteness}
Let $Y = (Y_1, \ldots Y_d)$ be a $d$-variate variable with finite first moment. We then have:
\begin{equation}
    \mathbbm{E}_{y \sim P(Y)}[\text{CCRPS}_\mathcal{T}(P(Y),y)] < \infty\text{.}
\end{equation}
\end{lemma}
\begin{proof}
We use the alternate but equivalent CRPS expression given by Gneiting and Raftery \cite{gneitingRaftery} to find an upper bound for the CRPS of a univariate distribution $D$ and observation $y$:
\begin{equation}
\text{CRPS}(D,y) = \mathbbm{E}_{x \sim D}|x-y| - \frac{1}{2}\mathbbm{E}_{x, x' \sim D}|x-x'| \leq \mathbbm{E}_{x \sim D}|x-y| \leq |y| + \mathbbm{E}_{x \sim D}|x|\text{.}
\end{equation}
Using this upper bound, and the finiteness of $X_{v_i}$'s first moment, we find:
\begin{align}
\begin{split}
&\mathbbm{E}_{y \sim P(Y)}[\text{CRPS}(P(Y_{v_i} \mid \forall_{j \in \mathcal{C}_i}: Y_j = y_j), y_{v_i})]ƒ\\ &\leq \mathbbm{E}_{y \sim P(Y_{v_i})}|y| + \mathbbm{E}_{y \sim P(Y)}[\mathbbm{E}_{z \sim P(Y_{v_i} \mid \forall_{j \in \mathcal{C}_i}: Y_j = y_j)}|z|] \\
&= \mathbbm{E}_{y \sim P(Y_{v_i})}|y| + \mathbbm{E}_{y \sim P(Y_{v_i})}|y| = 2\cdot \mathbbm{E}_{y \sim P(Y_{v_i})}|y|< \infty\text{.}
\end{split}
\end{align}
Lastly, using $\mathbbm{E}[\sum_{i=1}^q X_i] = \sum_{i=1}^q\mathbbm{E}[X_i]$ for $X \in \mathbbm{R}^d$, we find:
\begin{align}
\begin{split}
\mathbbm{E}_{y \sim P(Y)}[\text{CCRPS}_\mathcal{T}(P(Y),y)] = \sum_{i=1}^q \mathbbm{E}_{y \sim P(Y)}\left[ \text{CRPS}(P(Y_{v_i} \mid \forall_{j \in \mathcal{C}_i}: Y_j = y_j), y_{v_i})\right] < \infty\text{.}
\end{split}
\end{align}
\end{proof}
\subsection{Non-strict propriety of CCRPS}
We use Lemma \ref{applem:finiteness} to prove non-strict propriety.
\begin{lemma}[Propriety of Conditional CRPS]\label{lem:proper}
Let $\mathcal{T}=\{(v_i, \mathcal{C}_i)\}_{i=1}^q$ be a set of tuples, where $v_i \in \{1, ..., d\}$ and $\mathcal{C}_i \subseteq \{1, ..., d\} \setminus \{v_i\}$. Then $\text{CCRPS}_\mathcal{T}$ is proper for distributions with finite first moment. For all $d$-variate random variables $A, B$ we have:
\begin{equation}
\mathbbm{E}_{y \sim P(A)}[\text{CCRPS}_\mathcal{T}(P(A),y)] \leq \mathbbm{E}_{y \sim P(A)}[\text{CCRPS}_\mathcal{T}(P(B),y)]
\end{equation}
\end{lemma}
\begin{proof}
Consider two random variables $A = (A_1, \ldots, A_d)$ and $B = (B_1, \ldots, B_d)$ with finite first moment. Consider an arbitrary $\mathcal{T}$ as defined in the lemma statement. First, we expand the expected value for Conditional CRPS:
\begin{align}\label{eq:full_exp_A}
\begin{split}
&\mathbbm{E}_{y \sim P(A)}[\text{CCRPS}_\mathcal{T}(P(B),y)] \\&= \mathbbm{E}_{y \sim P(A)}\left[ \sum_{i=1}^q\text{CRPS}(P(A_{v_i} \mid \forall_{j \in \mathcal{C}_i}: A_j = y_j), y_{v_i})\right]\\
&= \sum_{i=1}^q \mathbbm{E}_{y \sim P(A)}\left[ \text{CRPS}(P(A_{v_i} \mid \forall_{j \in \mathcal{C}_i}: A_j = y_j), y_{v_i})\right]\\
&= \sum_{i=1}^q \mathbbm{E}_{y' \sim P(A_{\mathcal{C}_i})}\left[\mathbbm{E}_{z \sim P(A_{v_i} \mid \forall_{j \in \mathcal{C}_i}: A_j = y'_j)}\left[\text{CRPS}(P(A_{v_i} \mid \forall_{j \in \mathcal{C}_i}: A_j = y'_j), z)\right]\right]\text{.}
\end{split}
\end{align}
Here $P(A_{\mathcal{C}_i})$ denotes $P(A)$'s marginal distribution for the variables $(A_j)_{j \in \mathcal{C}_i}$. Similarly, we rewrite for $P(B)$:\footnote{Technically, if for any $1 \leq i \leq q$ we have a nonzero probability that $P(B_{v_i} \mid \forall_{j \in \mathcal{C}_i}: B_j = y_j)$ is ill-defined, we find $\mathbbm{E}_{y \sim A}[\text{CCRPS}_\mathcal{T}(B,y)] = \infty$, but this does not affect the statement made in Equation \eqref{appeq:non_strict_prop}.}
\begin{align}\label{eq:full_exp_B}
\begin{split}
&\mathbbm{E}_{y \sim P(A)}[\text{CCRPS}_\mathcal{T}(P(B),y)]\\
&= \sum_{i=1}^q \mathbbm{E}_{y' \sim P(A_{\mathcal{C}_i})}\left[\mathbbm{E}_{z \sim P(A_{v_i} \mid \forall_{j \in \mathcal{C}_i}: A_j = y'_j)}\left[\text{CRPS}(P(B_{v_i} \mid \forall_{j \in \mathcal{C}_i}: B_j = y'_j), z)\right]\right]\text{.}
\end{split}
\end{align}
Finally, noting finiteness of $\mathbbm{E}_{y \sim P(A)}[\text{CCRPS}_\mathcal{T}(P(A),y)]$ (Lemma \ref{applem:finiteness}), we apply univariate strict propriety of CRPS on each conditional term, and we find non-strict propriety:
\begin{equation}\label{appeq:non_strict_prop}
    \mathbbm{E}_{y \sim P(A)}[\text{CCRPS}_\mathcal{T}(P(A),y)] \leq \mathbbm{E}_{y \sim P(A)}[\text{CCRPS}_\mathcal{T}(P(B),y)]
\end{equation}
\end{proof}

\subsection{Discrete strict propriety of CCRPS}
We first introduce a helper lemma, which is used later in the strict propriety proofs.
\begin{lemma}[Probability of difference in conditionals and strict propriety]\label{lem:strict_proper_req}
Let $\mathcal{T}=\{(v_i, \mathcal{C}_i)\}_{i=1}^q$ be a set of tuples, where $v_i \in \{1, ..., d\}$ and $\mathcal{C}_i \subseteq \{1, ..., d\} \setminus \{v_i\}$. Let $\mathcal{D}$ be a set of $d$-variate distributions. Assume that for all distinct distributions $P(A), P(B) \in \mathcal{D}$ over random variables $A = (A_1, \ldots A_d)$ and $B = (B_1, \ldots, B_d)$ we have:\footnote{We denote the probability of an event $q$ by $\mathbbm{P}(q)$, and the probability distribution over a random variable $Q$ by $P(Q)$. Let $P(A_{\mathcal{C}_i})$ denote the marginal distribution for the variable $(A_j)_{j \in \mathcal{C}_i}$.Equation \eqref{eq:strict_prop_statement} thus denotes the chance that two conditional distributions are different.}
\begin{equation}\label{eq:strict_prop_statement}
    \mathbbm{P}_{y \sim P(A)}\left[P(A_{v_i} \mid \forall_{j \in \mathcal{C}_i}: A_j = y_j) \neq P(B_{v_i} \mid \forall_{j \in \mathcal{C}_i}: B_j = y_j)\right] > 0\text{,}
\end{equation}
that is, the two distributions will have one differing conditional with nonzero probability. Then $\text{CCRPS}_\mathcal{T}$ is strictly proper for $\mathcal{D}$.
\end{lemma}
\begin{proof}
Let $A = (A_1, \ldots, A_d)$ and $B = (B_1, \ldots, B_d)$ be random variables such that $P(A) \neq P(B)$. Let $P(A_{\mathcal{C}_i})$ denote the marginal distribution for the variable $(A_j)_{j \in \mathcal{C}_i}$. For every $1 \leq i \leq q$, we denote
\begin{equation}\label{eq:prob_greater_0}
    p_i = \mathbbm{P}_{y \sim P(A)}\left[P(A_{v_i} \mid \forall_{j \in \mathcal{C}_i}: A_j = y_j) \neq P(B_{v_i} \mid \forall_{j \in \mathcal{C}_i}: B_j = y_j)\right]
\end{equation}
Now assume $\exists_{1 \leq i \leq q}: p_i > 0$.

For all $1 \leq i \leq q$, let $U_i \subseteq \mathbbm{R}^{|\mathcal{C}_i|}$ denote the set of all marginal vectors $(y)_{j \in \mathcal{C}_i}$ such that $P(A_{v_i} \mid \forall_{j \in \mathcal{C}_i}: A_j = y_j) \neq P(B_{v_i} \mid \forall_{j \in \mathcal{C}_i}: B_j = y_j)$. We rewrite the expectancy of CCRPS to terms conditioned on the sets $U_i$. Using Equation \eqref{eq:full_exp_A}, we find:
\begin{align}\label{eq:rewrite_A}
\begin{split}
&\mathbbm{E}_{y \sim P(A)}[\text{CCRPS}_\mathcal{T}(P(A),y)] \\
&= \sum_{i=1}^q \mathbbm{E}_{y' \sim P(A_{\mathcal{C}_i})}\left[\mathbbm{E}_{z \sim P(A_{v_i} \mid \forall_{j \in \mathcal{C}_i}: A_j = y'_j)}\left[\text{CRPS}(P(A_{v_i} \mid \forall_{j \in \mathcal{C}_i}: A_j = y'_j), z)\right]\right]\\
    &= \sum_{i=1}^q p_i \cdot\mathbbm{E}_{y' \sim P(A) \cap y' \in U_i}\left[\mathbbm{E}_{z \sim P(A_{v_i} \mid \forall_{j \in \mathcal{C}_i}: A_j = y'_j)}\left[\text{CRPS}(P(A_{v_i} \mid \forall_{j \in \mathcal{C}_i}: A_j = y'_j), z)\right]\right]\\
    &+ \sum_{i=1}^q (1 - p_i) \cdot \mathbbm{E}_{y' \sim P(A) \cap y' \not\in U_i}\left[\mathbbm{E}_{z \sim P(A_{v_i} \mid \forall_{j \in \mathcal{C}_i}: A_j = y'_j)}\left[\text{CRPS}(P(A_{v_i} \mid \forall_{j \in \mathcal{C}_i}: A_j = y'_j), z)\right]\right]
\end{split}
\end{align}
and similarly we rewrite
\begin{align}\label{eq:rewrite_B}
\begin{split}
&\mathbbm{E}_{y \sim P(A)}[\text{CCRPS}_\mathcal{T}(P(B),y)] \\
    &= \sum_{i=1}^q p_i \cdot\mathbbm{E}_{y' \sim P(A) \cap y' \in U_i}\left[\mathbbm{E}_{z \sim P(A_{v_i} \mid \forall_{j \in \mathcal{C}_i}: A_j = y'_j)}\left[\text{CRPS}(P(B_{v_i} \mid \forall_{j \in \mathcal{C}_i}: B_j = y'_j), z)\right]\right]\\
    &+ \sum_{i=1}^q (1 - p_i) \cdot \mathbbm{E}_{y' \sim P(A) \cap y' \not\in U_i}\left[\mathbbm{E}_{z \sim P(A_{v_i} \mid \forall_{j \in \mathcal{C}_i}: A_j = y'_j)}\left[\text{CRPS}(P(B_{v_i} \mid \forall_{j \in \mathcal{C}_i}: B_j = y'_j), z)\right]\right]
\end{split}
\end{align}
For any $1 \leq i \leq q$, and any $y' \in U_i$, strict propriety of CRPS implies:
\begin{align}\label{eq:ineq_strict_prop}
    \begin{split}
        &\mathbbm{E}_{z \sim P(A_{v_i} \mid \forall_{j \in \mathcal{C}_i}: A_j = y'_j)}\left[\text{CRPS}(P(B_{v_i} \mid \forall_{j \in \mathcal{C}_i}: B_j = y'_j), z)\right]\\
        &< \mathbbm{E}_{z \sim P(A_{v_i} \mid \forall_{j \in \mathcal{C}_i}: A_j = y'_j)}\left[\text{CRPS}(P(B_{v_i} \mid \forall_{j \in \mathcal{C}_i}: B_j = y'_j), z)\right]
    \end{split}
\end{align}
and we find equality in Equation \eqref{eq:ineq_strict_prop} in the case that $y' \neq U_i$.
Therefore, using Equations \eqref{eq:rewrite_A} and \eqref{eq:rewrite_B}\begin{equation}
    \exists_{1 \leq i \leq q}: p_i > 0 \Longleftrightarrow \mathbbm{E}_{y \sim P(A)}[\text{CCRPS}_\mathcal{T}(P(A),y)] < \mathbbm{E}_{y \sim P(A)}[\text{CCRPS}_\mathcal{T}(B,y)]
\end{equation}
as we assumed existence of such $p_i > 0$, CCRPS is strictly proper for $\mathcal{D}$.
\end{proof}

Next, we will use Lemma \ref{lem:strict_proper_req} to prove strict propriety for discrete distributions.
\begin{lemma}[Strict propriety of Conditional CRPS for discrete distributions]\label{lem:discrete}
Let $\mathcal{T}=\{(v_i, \mathcal{C}_i)\}_{i=1}^q$ be a set of tuples, where $v_i \in \{1, ..., d\}$ and $\mathcal{C}_i \subseteq \{1, ..., d\} \setminus \{v_i\}$. Let $\phi_1, \ldots, \phi_d$ be a permutation of $\{1, \ldots, d\}$ such that:
\begin{equation}
\forall_{j=1}^d: (\phi_j, \{\phi_1, \ldots, \phi_{j-1}\}) \in \mathcal{T} \text{,}
\end{equation}
then $\text{CCRPS}_\mathcal{T}$ is strictly proper for discrete distributions with finite first moment, i.e. for all $d$-variate discrete distributions $P(A) \neq P(B)$ over random variables $A = (A_1, \ldots, A_d)$ and $B = (B_1, \ldots B_d)$ with finite first moment we have:
\begin{equation}
\mathbbm{E}_{y \sim P(A)}[\text{CCRPS}_\mathcal{T}(P(A),y)] < \mathbbm{E}_{y \sim P(A)}[\text{CCRPS}_\mathcal{T}(P(B),y)]
\end{equation}
\end{lemma}
\begin{proof}
Consider two discrete distributions $P(A) \neq P(B)$ over random variables $A = (A_1, \ldots, A_d)$ and $B = (B_1, \ldots B_d)$. Furthermore, let $v, \mathcal{C}$ and $\phi$ be defined as in the lemma statement. Without loss of generality, we assume $A$ and $B$ are defined over a countable set of events $\Omega = \omega^{(1)}, \omega^{(2)}, \ldots \in \mathbbm{R}^d$, i.e. $\sum_{y' \in \Omega}:\mathbbm{P}(A = y') = \sum_{y' \in \Omega}:\mathbbm{P}(B = y') = 1$.

We will prove that there exists $1 \leq i \leq q$ such that:
\begin{equation}\label{eq:contradictory_disc}
\mathbbm{P}_{y \sim A}\left[P(A_{v_i} \mid \forall_{j \in \mathcal{C}_i}: A_j = y_j) \neq P(B_{v_i} \mid \forall_{j \in \mathcal{C}_i}: B_j = y_j)\right] > 0\text{,}
\end{equation} 
which, by Lemma \ref{lem:strict_proper_req} proves strict propriety.

Since $P(A) \neq P(B)$, there exists an event to which $P(A)$ assigns a higher probability than $P(B)$, i.e. there exists an $s \in \mathbbm{N}_+$, such that $\mathbbm{P}_{y \sim P(A)}(y = \omega^{(s)}) > \mathbbm{P}_{y \sim P(B)}(y = \omega^{(s)})$. Since $P(A)$'s probability attributed to $y = \omega^{(s)}$ is greater than zero, we rewrite this probability via the chain rule:
\begin{equation}
\mathbbm{P}_{y \sim P(A)}(y = \omega^{(s)}) = \mathbbm{P}_{y_{\phi_1} \sim P(A_{\phi_1})}(y_{\phi_1} = \omega_{\phi_1}^{(s)}) \cdot \ldots \cdot \mathbbm{P}_{y_{\phi_d} \sim P(A_{\phi_d} | \forall_{j =1}^{d-1} A_{\phi_j} = y_{\phi_j})}(y_{\phi_d} = \omega_{\phi_d}^{(s)})\text{.}
\end{equation}
Now, we consider two cases:\\
\\
\noindent\textbf{Case 1:} At least one of $P(B)$'s conditionals $P(B_{\phi_k} | \forall_{j =1}^{k-1} B_{\phi_j} = y_{\phi_j})$ is ill-defined. We find that \begin{equation}
\text{CRPS}(P(B_{\phi_k} | \forall_{j =1}^{k-1} B_{\phi_j} = y_{\phi_j}), y_{\phi_k}) = \infty\text{.}
\end{equation}Since $\mathbbm{P}_{y \sim P(A)}(y = \omega_s) > 0$, we find:
\begin{equation}
    \mathbbm{E}_{y \sim P(A)}[\text{CRPS}(P(B_{\phi_k} | \forall_{j =1}^{k-1} B_{\phi_j} = y_{\phi_j}), y_{\phi_k})] = \infty\text{.}
\end{equation}
Since $(\phi_k, \{\phi_1, \ldots, \phi_{k-1}\}) \in \mathcal{T}$, this conditional is included in $\text{CCRPS}_\mathcal{T}$, and we naturally find 
\begin{equation}
    \mathbbm{E}_{y \sim A}[\text{CCRPS}_\mathcal{T}(P(B), y)] = \infty\text{.}
\end{equation}
Finiteness of $\mathbbm{E}_{y \sim P(A)}[\text{CCRPS}_\mathcal{T}(P(A), y)]$ (Lemma \ref{applem:finiteness}) now proves strict propriety.\\
\\
\noindent\textbf{Case 2:} All of $P(B)$'s conditionals are well-defined. Then we can apply a similar chain-rule decomposition. $\mathbbm{P}_{y \sim P(A)}(y = \omega^{(s)}) > \mathbbm{P}_{y \sim P(B)}(y = \omega^{(s)})$ now implies existence of $1 \leq k \leq d$ such that:
\begin{equation}
    \mathbbm{P}_{y \sim P(A_{\phi_k} | \forall_{j =1}^{k-1} A_{\phi_j})}(y = \omega^{(s)}_{\phi_k}) > \mathbbm{P}_{y \sim P(B_{\phi_k} | \forall_{j =1}^{k-1} B_{\phi_j})}(y = \omega^{(s)}_{\phi_k})
\end{equation}
Since the two conditionals assign different probabilities to the same event, we conclude:
\begin{equation}
P(A_{\phi_k} | \forall_{j =1}^{k-1} A_{\phi_j}) \neq P(B_{\phi_k} | \forall_{j =1}^{k-1} B_{\phi_j})\text{.}
\end{equation}
Please note that $\mathbbm{P}_{y \sim P(A)}(y = \omega^{s}) > 0$ implies $\mathbbm{P}_{y \sim P(A)}(\forall_{j = 1}^{k-1}: y_{\phi_j}= \omega_{\phi_j}^{s}) > 0$. Since $(\phi_k, \{\phi_1, \ldots, \phi_{k-1}\}) \in \mathcal{T}$,there exists $1 \leq i \leq q$ such that:
\begin{equation}
    \mathbbm{P}_{y \sim P(A)}\left[P(A_{v_i} \mid \forall_{j \in \mathcal{C}_i}: A_j = y_j) \neq P(B_{v_i} \mid \forall_{j \in \mathcal{C}_i}: B_j = y_j)\right] > 0\text{.}
\end{equation}
Therefore, by Lemma \ref{lem:strict_proper_req}, we prove strict propriety.
\end{proof}
\subsection{Absolutely continuous strict propriety of CCRPS}
The strict propriety proof for absolutely continous distributions is similarly structured as the proof of Lemma \ref{lem:discrete} and in many ways a continuous equivalent.
\begin{lemma}[Strict propriety of Conditional CRPS for absolutely continuous distributions]\label{lem:cont}
Let $\mathcal{T}=\{(v_i, \mathcal{C}_i)\}_{i=1}^q$ be a set of tuples, where $v_i \in \{1, ..., d\}$ and $\mathcal{C}_i \subseteq \{1, ..., d\} \setminus \{v_i\}$. Let $\phi_1, \ldots, \phi_d$ be a permutation of $\{1, \ldots, d\}$ such that:
\begin{equation}
\forall_{j=1}^d: (\phi_j, \{\phi_1, \ldots, \phi_{j-1}\}) \in \mathcal{T}\text{,}
\end{equation}
then $\text{CCRPS}_\mathcal{T}$ is strictly proper for absolutely continuous distributions with finite first moment, i.e. for all $d$-variate absolutely continuous distributions $P(A) \neq P(B)$ with finite first moment we have:
\begin{equation}
\mathbbm{E}_{y \sim P(A)}[\text{CCRPS}_\mathcal{T}(P(A),y)] < \mathbbm{E}_{y \sim P(A)}[\text{CCRPS}_\mathcal{T}(P(B),y)]
\end{equation}
\end{lemma}
\begin{proof}
Consider two absolutely continuous distributions $P(A) \neq P(B)$ with finite first moment, defined over random variables $A = (A_1, ..., A_d)$ and $B = (B_1, ..., B_d)$ respectively. Let $v, \mathcal{C}$ and $\phi$ be defined as in the lemma statement. Since $P(A)$ and $P(B)$ are absolutely continuous, they both have Lebesgue integratable probability density functions, which we denote by $f_{P(A)}$ and $f_{P(B)}$.

We will prove that there exists $1 \leq i \leq q$ such that:
\begin{equation}\label{eq:contradictory}
\mathbbm{P}_{y \sim P(A)}\left(P(A_{v_i} \mid \forall_{j \in \mathcal{C}_i}: A_j = y_j) \neq P(B_{v_i} \mid \forall_{j \in \mathcal{C}_i}: B_j = y_j)\right) > 0\text{,}
\end{equation} 
which, by Lemma \ref{lem:strict_proper_req} is enough to prove strict propriety.

Since $P(A) \neq P(B)$, there exists a set $U \subset \mathbbm{R}^d$ to which $P(A)$ assigns a larger probability than $P(B)$. Due to absolute continuity of both $P(A)$ and $P(B)$, this can be expressed via integrals over their density functions:
\begin{equation}\label{appeq:diffprob}
    \mathbbm{P}_{y \sim P(A)}(y \in U) = \int_{U} f_{P(A)}(u) du > \int_{U} f_{P(B)}(u) du = \mathbbm{P}_{y \sim P(B)}(y \in U)
\end{equation}
Now, consider the subset of all points in $U$ for which $P(A)$'s density is greater than $P(B)'s$ density:  $U' = \{{u \in U}\mid  f_{P(A)}(u) - f_{P(B)}(u) > 0\}$. Trivially, Equation \eqref{appeq:diffprob} implies:\footnote{Since $f_A - f_B$ is a Lebesgue measurable function, $U'$ is Lebesgue measurable.}
\begin{equation}
    \int_{U'} du > 0\text{.}
\end{equation}
Next, since for all $y \in U'$ we have $f_{P(A)}(y) > 0$, we can rewrite the density of $A$ to the following conditional densities via the chain rule:
\begin{equation}\label{eq:chain_rule}
f_{P(A)}(y) = f_{P(A_{\phi_1})}(y_{\phi_1}) \cdot f_{P(A_{\phi_1} | A_{\phi_1} = y_{\phi_1})} (y_{\phi_2}) \cdot \ldots \cdot f_{P(A_{\phi_d} | \forall_{j=1}^{d-1} A_{\phi_j} = y_{\phi_j})} (y_{\phi_d})\text{.}
\end{equation}
Next, we consider the same conditional decomposition for $f_{P(B)}$. However, since we do not know $f_{P(B)}(y) > 0$, this conditional decomposition might not always be well-defined. Therefore, let $\hat U \subseteq U'$ be the subset of all $y \in U'$ such that at least one of $P(B_{\phi_1}), \ldots, P(B_{\phi_d} | \forall_{j=1}^{d-1} A_{\phi_j} = y_{\phi_j})$ is ill-defined. For each $y \in \hat U$, we know that there exists $1 \leq i \leq d$ such that (by CCRPS definition):
\begin{equation}\text{CRPS}(P(B_{\phi_i} | \forall_{j=1}^{i-1} A_{\phi_j} = y_{\phi_j}), y_{\phi_i}) = \infty
\end{equation}
We will now consider two cases:\footnote{Since the marginal density function $P(A_{\phi_{1}}, \ldots, A_{\phi_{d-1}})$ is also Lebesgue measurable, and we can write $\hat U$ as the superlevel set of this function, $
    \hat U = \{u \in U' \mid f_{A_{\phi_{1}}, \ldots, A_{\phi_{d-1}}}(u_{\phi_1}, \ldots, u_{\phi_{d-1}}) > 0 \}$, the integral $\int_{\hat U} du$ is defined.}\\
\\
\noindent\textbf{Case 1:} $\int_{\hat U} dy > 0$. In this case, as $f_{P(A)}$ is strictly positive in $\hat U$, we find
\begin{equation}
\mathbbm{P}_{y \sim P(A)}(y \in \hat U) = \int_{\hat U} f_{P(A)}(y) dy > 0\text{.}
\end{equation}Furthermore, since $\forall_{j=1}^d: (\phi_j, \{\phi_1, \ldots, \phi_{j-1}\}) \in \mathcal{T}$, this conditional is evaluated in $\text{CCRPS}_\mathcal{T}$. Therefore we find for all $y \in \hat U$:
\begin{equation}
    \text{CCRPS}_\mathcal{T}(P(B),y) = \infty
\end{equation}
and thus $\mathbbm{E}_{y \sim P(A)}[\text{CCRPS}_\mathcal{T}(P(B),y)] = \infty$. Finiteness of $\mathbbm{E}_{y \sim P(A)}[\text{CCRPS}_\mathcal{T}(P(A),y)]$ (Lemma \ref{applem:finiteness}) then proves the lemma.\\
\\
\noindent\textbf{Case 2:} Let $\int_{\hat U} dy = 0$. Then we define $\overline U = U' \setminus \hat U$: the set of all $y \in U'$ for which the conditional decomposition of $P(B)$ is well-defined. Naturally, $\int_{\overline U} dy = \int_{U'} dy - \int_{\hat U} dy> 0$. We know that for all $y \in \overline U$, there exists $1 \leq i \leq q$ such that:
\begin{equation}\label{eq:ineq}
f_{P(A_{\phi_i} | \forall_{j=1}^{i-1} A_{\phi_j} = y_{\phi_j})} (y_{\phi_i}) > f_{P(B_{\phi_i} | \forall_{j=1}^{i-1} B_{\phi_j} = y_{\phi_j})} (y_{\phi_i})\text{.} 
\end{equation}
Please note that the value of $i$ for which Equation \eqref{eq:ineq} holds, may change depending on the selection of $y \in \overline U$. However, since only a finite number of values for $i$ are considered, and $\mathbbm{P}_{y \sim P(A)}(y \in \overline U) > 0$, there must exist a $1 \leq k \leq d$ and a subset $U^* \subseteq \overline U$ with $\mathbbm{P}_{y \sim P(A)}(y \in U^*) > 0$ such that for all $y \in U^*$, Equation \eqref{eq:ineq} holds:\footnote{In this case, $f_{P(A_{\phi_i} | \forall_{j=1}^{i-1} A_{\phi_j} = y_{\phi_j})} (y_{\phi_i}) - f_{P(B_{\phi_i} | \forall_{j=1}^{i-1} B_{\phi_j} = y_{\phi_j})} (y_{\phi_i})$ is Lebesgue measurable, and we similarly deduce that $\mathbbm{P}_{y \sim P(A)}(y \in U^*)$ and $\int_{U^*} du$ are well-defined.}
\begin{equation}\label{eq:k_ineq}
    f_{P(A_{\phi_k} | \forall_{j=1}^{k-1} A_{\phi_j} = y_{\phi_j})} (y_{\phi_k}) > f_{P(B_{\phi_k} | \forall_{j=1}^{k-1} B_{\phi_j} = y_{\phi_j})} (y_{\phi_k}) \text{.}
\end{equation}

Since $U^*$, is non-zero volumed, it contains an open set. That means it contains a hypercube $H = [\hat y_j- \epsilon, \hat y_j + \epsilon]_{j=1}^d \subseteq U^*$ for some $\hat y \in U^*$ and $\epsilon > 0$. For ease of notation, we denote $H$'s marginal projection on the $j$'th dimension by $H_j = [\hat y_j- \epsilon, \hat y_j + \epsilon]$.

We now show that for all $y$ such that $\forall_{j \neq \phi_k}: y_j \in H_j$, we have $P(A_{\phi_k} \mid \forall_{j=1}^{k-1} A_{\phi_j} = y_{\phi_j}) \neq P(B_{\phi_k} \mid \forall_{j=1}^{k-1} B_{\phi_j} = y_{\phi_j})$. Using the inequality of the conditional densities in $H$ as described in Equation \eqref{eq:k_ineq}, we find:
\begin{align}
\begin{split}
    \mathbbm{P}_{z \sim P(A_{\phi_k} \mid \forall_{j=1}^{k-1} A_{\phi_j} = y_{\phi_j})}(z \in H_{\phi_k}) & = \int_{y_{\phi_k}' - \epsilon}^{y_{\phi_k}' + \epsilon} f_{P(A_{\phi_k} \mid \forall_{j=1}^{k-1} A_{\phi_j} = y_{\phi_j})}(h) dh\\
    &< \int_{y_{\phi_k}' - \epsilon}^{y_{\phi_k}' + \epsilon} f_{P(B_{\phi_k} \mid \forall_{j=1}^{k-1} B_{\phi_j} = y_{\phi_j})}(h) dh\\
    &=\mathbbm{P}_{z \sim P(B_{\phi_k} \mid \forall_{j=1}^{k-1} B_{\phi_j} = y_{\phi_j})}(z \in H_{\phi_k})
\end{split}
\end{align}
Hence  if $\forall_{j \neq \phi_k}: y_j \in H_j$, then $P(A_{\phi_k} \mid \forall_{j=1}^{k-1} A_{\phi_j} = y_{\phi_j})$ and $P(B_{\phi_k} \mid \forall_{j=1}^{k-1} B_{\phi_j} = y_{\phi_j})$ assign different probabilities to the same event, thus they are different distributions.\\
\\
\noindent Finally, we will prove such event $\forall_{j \neq \phi_k}: y_j \in H_j$ happens with nonzero probability. Since $H \subseteq U'$, we know:
\begin{equation}
    \mathbbm{P}_{y \sim P(A)}(y \in H) = \int_H f_{P(A)}(y) > 0
\end{equation}
and since $y \in H \implies \forall_{j \neq \phi_k}: y_j \in H_j$, we also know $\mathbbm{P}_{y \sim P(A)}(\forall_{j \neq \phi_k}: y_j \in H_j)$. Therefore, we find:
\begin{equation}
    \mathbbm{P}_{y \sim P(A)}(P(A_{\phi_k} \mid \forall_{j=1}^{k-1} A_{\phi_j} = y_{\phi_j}) \neq P(B_{\phi_k} \mid \forall_{j=1}^{k-1} B_{\phi_j} = y_{\phi_j})) > 0\text{.}
\end{equation}
Since $(\phi_k, \{\phi_1, \ldots, \phi_{k-1}\}) \in \mathcal{T}$, we find that there exists $1 \leq i \leq q$ such that:
\begin{equation}
    \mathbbm{P}_{y \sim P(A)}\left(P(A_{v_i} \mid \forall_{j \in \mathcal{C}_i}: A_j = y_j) \neq P(B_{v_i} \mid \forall_{j \in \mathcal{C}_i}: B_j = y_j)\right) > 0
\end{equation}
Hence, by Lemma \ref{lem:strict_proper_req}, we prove strict propriety.
\end{proof}

\subsection{Counter example for partially continuous distributions}
Finally, we give a counter example, to show that Conditional CRPS is never strictly proper for all distributions with finite first moment, regardless of our choice of conditional specification $\mathcal{T}$. 
\begin{lemma}
There exist two $d$-variate distributions $P(A) \neq P(B)$ with finite first moment, such that
\begin{equation}
    \mathbbm{E}_{y \sim P(A)}[\text{CCRPS}_\mathcal{T}(P(A),y)] =     \mathbbm{E}_{y \sim P(A)}[\text{CCRPS}_\mathcal{T}(P(B),y)]\text{,}
\end{equation}
regardless of our choice for $\mathcal{T}$.
\end{lemma}
\begin{proof}
Consider $P(A) = \mathcal{N}_d(0, I)$, a $d$-variate standard normal distribution over random variables $A_1, \ldots A_d$. Next, consider $P(B)$ (over random variables $B_1, \ldots B_d$), which is defined as follows: with $0.5$ probability, $P(B)$ samples a vector $(z, \ldots, z) \in \mathbbm{R}^d$ with $z \in \mathcal{N}_1(0,1)$. Otherwise, $P(B)$ samples a vector i.i.d. to $P(A)$.\\
\\
\noindent Consider an arbitrary conditional specification $v_1, ..., v_q \in \{1, ..., d\}$ and $\mathcal{C}_1, ..., \mathcal{C}_q$ with $\mathcal{C}_i \subseteq \{1, ..., d\} \setminus \{v_i\}$. By Lemma \ref{lem:strict_proper_req}, it suffices to show that for any $1 \leq i \leq q$ we have:
\begin{equation}\label{eq:counter_example_to_prove}
    \mathbbm{P}_{y \sim P(A)}(P(A_{v_i} \mid \forall_{j \in \mathcal{C}_i}: A_j = y_j) \neq P(B_{v_i} \mid \forall_{j \in \mathcal{C}_i}: B_j = y_j)) = 0
    \end{equation}
Let us first start with the case that $\mathcal{C}_i = \emptyset$, i.e. marginal distributions. We find:
\begin{equation}
    f_{P(B_{v_i})}(y) = 0.5 \cdot f_{A_{v_i}}(y) + 0.5 \cdot f_{\mathcal{N}_1(0,1)}(y) = f_{\mathcal{N}_1(0,1)}(y) = f_{P(A_{v_i})}(y){,}
\end{equation}
thus $P(A_{v_i}) = P(B_{v_i})$, and Equation \eqref{eq:counter_example_to_prove} holds trivially.\\
\\
\noindent Secondly, let us now consider the case $\mathcal{C}_i \neq \emptyset$. Then we find:
\begin{equation}
    P(B_{v_i} \mid \forall_{j \in \mathcal{C}_i}: B_j = y_j) = 
    \begin{cases} P(z)\text{,} & \text{if } \exists_{z \in \mathbbm{R}}:\forall_{j \in \mathcal{C}_i}: y_j = z\\
    \mathcal{N}_1(0, 1)\text{,} & \text{otherwise}  \\
    \end{cases}
\end{equation}
Here, $P(z)$ denotes the probability distribution with CDF $F_{P(z)}(x) = \mathbbm{1}(x \leq z)$. With probability 1, any $y$ sampled from $A$ does not have equal dimensional variables, and we have $P(B_{v_i} \mid \forall_{j \in \mathcal{C}_i}: B_j = y_j) = \mathcal{N}_1(0, 1) = P(A_{v_i} \mid \forall_{j \in \mathcal{C}_i}: A_j = y_j)$. Therefore, we find:
\begin{equation}
    \mathbbm{P}_{y \sim P(A)}(P(A_{v_i} \mid \forall_{j \in \mathcal{C}_i}: A_j = y_j) \neq P(B_{v_i} \mid \forall_{j \in \mathcal{C}_i}: B_j = y_j)) = 0
\end{equation}
And thus Equation \eqref{eq:counter_example_to_prove} holds.
\end{proof}

\section{Expressions for Conditional CRPS}
As full expressions of Conditional CRPS can get quite lenghty, in this Appendix, we will give an overview on univariate conditional and marginal distributions of popular multivariate distributions, and their CRPS expressions. Since Conditional CRPS consists of summations over such CRPS terms (based on $\mathcal{T}$), this allows for easy derivation of conditional CRPS expressions.

\subsection{Multivariate Gaussian distribution}\label{sec:ccrps_gaus}
Consider a $d$-variate Gaussian distribution $\mathcal{N}_d(\mu, \Sigma)$ over random variables $A_1, \ldots, A_d$, with $\mu \in \mathbbm{R}^d$ and $\Sigma \in \mathbbm{R}^{d \times d}$ being positive definite. To be able to compute $P(A_i | \hat A = \hat a)$ for a group of variables $\hat A$, we will denote the mean vector and covariance matrix via block notation:
\begin{equation}\label{eq:block_notation}
    \overline \mu = \begin{bmatrix}
\mu_i\\
\mu_{\hat A}\\
\end{bmatrix},  \overline \Sigma = \begin{bmatrix}
\sigma^2_i & \rho\\
\rho^T & \Sigma_{\hat A}
\end{bmatrix}\text{.}
\end{equation}
Here, we have left out irrelevant variables. We then find that any univariate conditional or marginal is univariate Gaussian distributed
\begin{equation}\label{eq:GausCond}
    P(A_i| \hat A = \hat a) = \mathcal{N}_1\left(\mu_i + \rho \Sigma_{\hat A}^{-1}(\hat a - \mu_{\hat A}), \sigma^2_i - \rho \Sigma_{\hat A}^{-1}\rho^{T} \right)\text{.}
\end{equation}
In the special case that $\hat A = \emptyset$ we find $P(A_i) = \mathcal{N}_1(\mu_i, \sigma_i^2)$. A closed-form CRPS expression for univariate Gaussian distributions has been provided by Gneting et al. \cite{GneitingRaftery2}:
\begin{equation}\label{eq;CRPS:Gaussian}
    \text{CRPS}(\mathcal{N}_1(\mu, \sigma^2), y) = \sigma \cdot \left(\frac{y-\mu}{\sigma} \cdot \left(2\Phi(\frac{y-\mu}{\sigma}) - 1\right) + 2 \varphi(\frac{y-\mu}{\sigma}) - \frac{1}{\sqrt{\mathcal{C}}}\right)
\end{equation}
Here, $\varphi$ and $\Phi$ are the PDF and CDF of a standard normal distribution. 

\subsection{Multivariate mixture Gaussian distribution}\label{sec:ccrps_mixt}
 Let $A$ be a $d$-variate mixture Gaussian distribution over random variables $A_1, \ldots A_d$. $A$ consists of $m$ multivariate Gaussian distributions $D^{(1)}, \ldots, D^{(m)}$. We write the density as a weighted sum of the densities of the (multivariate Gaussian) mixture members:
\begin{equation}\label{eq:MixGauss1}
    f_A(x) = \sum_{j=1}^m \lambda^{(j)} \cdot f_{D^{(j)}}(x)\text{.}
\end{equation}
Rewriting this density to an expression for the conditional density via the chain rule, we find the following density for the conditional distribution of a variable $A_i$ given an observation of $\hat a$ a group of variables $\hat A$. If we similarly define $\hat D^{(j)}$ and $D_i^{(j)}$ to be random variables of the $j$'th mixture corresponding to $\hat A$ and $A_i$ respectively, then we find:
\begin{align}\label{eq:MixGausCond}
    \begin{split}
        f_{P(A_i | \hat A = \hat a)}(x) &=\sum_{j=1}^m \frac{\lambda^{(j)} \cdot f_{P(\hat D^{(j)})}(\hat a)}{f_{P(\hat A)}(\hat a)}\cdot f_{P(D^{(j)}_i | \hat D^{(j)} = \hat a)}(x)
    \end{split}\text{.}
\end{align}
 or in the case that $\hat A = \emptyset$, we find: $f_{P(A_i | \hat A = \hat a)}(x) =\sum_{j=1}^m \lambda^{(j)} \cdot f_{P(D^{(j)}_i)}(x)$. Hence, marginal and conditional distributions are mixtures of univariate Gaussian distributons, with parameters of $P(D^{(j)}_i | \hat D^{(j)} = \hat a)$ and $P(D^{(j)}_i)$ given via Equation \eqref{eq:GausCond}. The parameters found in Section \ref{sec:ccrps-mixt-gaus} are a special case of Equations \eqref{eq:MixGausCond} and \eqref{eq:GausCond} where $|\hat A| \leq 1$. Grimit et al. \cite{grimit} gave a closed-form expression of such distributions in Equation \eqref{eq:crps_grimit}.

 \subsection{Multivariate Log-normal distribution}\label{sec:ccrps_LN}
Let $A$, defined over random variables $A_1, \ldots A_d$ be a $d$-variate log-normal distribution with paramers $\mu \in \mathbbm{R}^d$ and $\Sigma \in \mathbbm{R}^{d \times d}$. That is, $(\log(A_1), \ldots, \log(A_d) \sim \mathcal{N}_{d}(\mu, \Sigma)$. Similarly to Gaussian distributions, the marginal and conditional distributions of a multivariate Log-normal distribution are univariate Log-normally distributed. Re-using the notation from Appendix \ref{sec:ccrps_gaus}, we find:
\begin{equation}\label{eq:LNCond}
    P(A_i| \hat A = \hat a) = LN_1\left(\mu_i + \rho \Sigma_{\hat A}^{-1}(\log (\hat a) - \mu_{\hat A}), \sigma^2_i - \rho \Sigma_{\hat A}^{-1}\rho^{T} \right)\text{.}
\end{equation}
A closed-form CRPS expression for univariate Log-normal distributions is given by Baran and Lerch \cite{BaranLerch}:
\begin{equation}\label{eq:CRPS:Logormal}
    \text{CRPS}(LN(\mu, \sigma^2), y) = \begin{cases}
y\left(2 \Phi\left(y_0\right)-1\right) - 2e^{\mu+\sigma^2/2} \left( \Phi\left(y_0-\sigma\right)+\Phi\left(\frac{\sigma}{\sqrt{2}}\right)-1\right) & \text{if } y > 0\\
2e^{\mu+\sigma^2/2} \left(\Phi\left(1 - \frac{\sigma}{\sqrt{2}}\right)\right) -y &\text{if } y \leq 0
\end{cases}
\end{equation}
Here, $y_0 = \frac{\log y - \mu}{\sigma}$ and $\Phi$ is the CDF of a standard normal distribution.
\subsection{Multivariate student-t distribution}
Let $A = t(\mu, \Sigma, \nu)$ be a $d$-variate student t-distribution, defined by the following density function \cite{roth2012multivariate}:
\begin{equation}
    f_A(x) = \frac{\Gamma[(\nu + d)/2]}{\Gamma(\nu/2)\nu^{d/2} \mathcal{C}^{d/2}}\left(1 + \frac{1}{\nu}(x - \mu)^\top \Sigma^{-1}(x - \mu)\right)^{-(\nu + d)/2}
\end{equation}
Re-using the block notation from Appendix \ref{sec:ccrps_gaus}, the conditional distribution $P(A_i | \hat A = \hat a)$ of a variable $A_i$ given a group of other variables $\hat A$ is given by a univariate student t-distribution with parameters given by Ding \cite{student-t-conditionals}:
\begin{equation}\label{eq:student-t-conditional}
    P(A_i | \hat A = \hat a) = t\left(\mu_{\hat A} + \rho \Sigma_{\hat A}^{-1}(\hat a - \mu_{\hat A}), \frac{\nu + s}{\nu + |\hat A|} \sigma^2_i - \rho \Sigma_{\hat A}^{-1}\rho^T, \nu + |\hat A|\right)
\end{equation}
here, $s = (\hat a - \mu_{\hat A})^\top \Sigma_{\hat A}^{-1} (\hat a - \mu_{\hat A})$. The CRPS of such distribution is then given by Alexander et al. \cite{JSSv090i12}:
\begin{equation}
    \text{CRPS}(t(\mu, \sigma, \nu),y) = \sigma \left(y_0 (2 F_\nu(y_0) - 1) + 2 f_\nu(y_0) \left(\frac{\nu + y_0^2}{\nu - 1}\right) - \frac{2\sqrt{\nu}B(\frac{1}{2}, \nu - \frac{1}{2})}{(\nu - 1)B\left(\frac{1}{2},\frac{\nu}{2}\right)^2}\right)
\end{equation}
Here, $y_0 = \frac{y - \mu}{\sigma}$. $f_\nu(x)$ and $F_\nu(x)$ are the PDF and CDF of the standard student-t distribution with $\nu$ degrees of freedom.

\subsection{Dirichlet distribution}
Consider a Dirichlet distribution, $A = \text{Dir}(\alpha_1, ..., \alpha_d)$, is defined on a Euclidean space $\mathbbm{R}^{d-1}$ and given by its density function:
\begin{equation}
    f_A(x_1, ..., x_d) = \frac{\Gamma(\sum_{i=1}^d \alpha_i)}{\prod_{i=1}^d \Gamma(\alpha_i)}\prod_{i=1}^d x_i^{\alpha_i-1}
\end{equation}
Here $\Gamma(x)$ is the Gamma function, for all $i$ we have, $x_i \in [0,1]$, and $\sum_{i=1}^d x_i = 1$. Let $\alpha_0 = \sum_{j=1}^d \alpha_j$. Using the additive property of Dirichlet distributions, we find that its $i$'th marginal is simply given by Beta distributions: $P(A_i) = B(\alpha_i, \alpha_0 - \alpha_i)$, and similarly, its univariate conditionals have scaled Beta distributions, i.e, for the distribution $P(A_i | \hat A = \hat a)$ for a variable $A_i$ given an observation $\hat a$ of a group of other variables $\hat A$.
\begin{equation}\label{eq:scaled_beta}
\frac{1}{1 - \sum_j \hat a_j} A_i | \hat A = \hat a \sim B\left(\alpha_i, \alpha_0 - \alpha_i\right)
\end{equation}
The CRPS of a Beta distribution is given by Taillardat et al. \cite{crps_qrf}:
\begin{align}\label{eq:CRPS:beta}
\begin{split}
\text{CRPS}(B(\alpha, \beta), y) & = \frac{\alpha}{\alpha+\beta}\left(1 - \Phi(y; \alpha+1, \beta)\right) - y\left(1 - 2\Phi(y; \alpha, \beta)\right)- \frac{1}{\alpha+\beta} \frac{\Gamma(\alpha+\beta)\Gamma\left(\alpha+\frac{1}{2}\right)\Gamma\left(\beta+\frac{1}{2}\right)}{\sqrt{\mathcal{C}} \Gamma\left(\alpha+\beta+\frac{1}{2}\right)\Gamma(\alpha)\Gamma(\beta)}
\end{split}
\end{align}
Here, $\Gamma(x)$ is the Gamma function, and $\Phi(x;\alpha, \beta)$ is the CDF of a Beta distribution with parameters $\alpha$ and $\beta$, extended with $\Phi(x;\alpha,\beta) = 0$ for $x < 0$ and $\Phi(x;\alpha,\beta) = 1$ for $x > 1$. The CRPS of a scaled beta distribution given in Equation \eqref{eq:scaled_beta} is simply given by rescaling:
\begin{equation}
\text{CRPS}(P(A_i | \hat A = \hat a), y) = \left(1 - \sum_j \hat a_j\right) \cdot \text{CRPS}\left(B(\alpha_i, \alpha_0 - \alpha_i), \frac{y}{1 - \sum_j \hat a_j}\right)\text{.}
\end{equation}

\section{Approximations of the Energy Score and Variogram Score}
In this appendix, we will specify the approximations and analytic formulas used to compute the Variogram Scores and Energy Scores presented in our experiments. First, recall the Energy Score \cite{gneitingRaftery}:
\begin{equation}\label{eq:ES}
    \text{ES}(A, y) = \mathbbm{E}_{x \sim A}[\lVert x - y \rVert]_2^\beta - \frac{1}{2}\mathbbm{E}_{x, x' \sim A}[\lVert x - x' \rVert]_2^\beta\text{.}
\end{equation}
For (weighted) ensemble distributions, i.e. distributions defined as a set of vectors $x_1, \ldots, x_m$ with probabilities $w_1, \ldots, w_m$ such that $\sum_{i=1}^m w_i = 1$, we find rather simple expression for the Energy Score:
\begin{equation}\label{eq:es_approx}
    \text{ES}(A, y) \approx \sum_{i=1}^m w_i \lVert x_i - y\rVert_2^\beta - \frac{1}{2}\sum_{i,j=1}^d w_i w_j \lVert x_i - x_j\rVert_2^\beta
\end{equation}
For multivariate mixture Gaussian distributions, we approximated the Energy score by Equation \eqref{eq:es_approx}, sampling $v_1, \ldots, v_m \sim A$, and $\forall_{i=1}^m: w_i = \frac{1}{m}$.

Next, recall the Variogram Score \cite{variogram_score}:
\begin{equation}
    \text{VarS}_{p}(A, y) = \sum_{i < j}^d (|y_i - y_j|^p - \mathbbm{E}_{x \sim A}\left[|x_i - x_j|^p\right])^2\text{.}
\end{equation}
For (weighted) ensemble distributions, we applied the following approximation for $\mathbbm{E}_{x \sim A}\left[|x_i - x_j|^p\right]$:
\begin{equation}
\mathbbm{E}_{x \sim A}\left[|x_i - x_j|^p\right] \approx \sum_{j=1}^d w_i \cdot | v_i - v_j|^p\text{,}
\end{equation}
For mixtures of $d$-variate Gaussian distributions, defined over weights $\lambda^{(1)}, \ldots \lambda^{(m)} \in [0,1]$, covariance matrices $\Sigma^{(1)}, \ldots , \Sigma^{(m)} \in \mathbbm{R}^{d \times d}$ and mean vectors $\mu^{(1)}, \ldots , \mu^{(m)} \in \mathbbm{R}^{d}$ we found the following expression for $\mathbbm{E}_{x \sim A}\left[|x_i - x_j|^p\right]$, based on work by Winkelbauer \cite{abs_moments_Gaussian}:
\begin{equation}
\mathbbm{E}_{x \sim A}\left[|x_i - x_j|^p\right] = \sum_{k=1}^m \lambda_i \cdot \hat \sigma_{ijk}^p \cdot 2^{p/2} \cdot \frac{\Gamma((p+1)/2)}{\sqrt{\mathcal{C}}} \cdot {}_1F_1\left(-\frac{p}{2}, \frac{1}{2}; -\frac{\hat \mu_{ijk}^2}{2\hat \sigma_{ijk}^2}\right)\text{.}
\end{equation}
Here, $\hat \sigma_{ijk} = \sqrt{\Sigma_{ii}^{(k)} + \Sigma_{jj}^{(k)} - 2\Sigma_{ij}^{(k)}}$, $\hat \mu_{ijk} = \mu_i^{(k)} - \mu_j^{(k)}$, $\Gamma(.)$ denotes the Gamma function and $_1 F_1$ denotes Kummer's confluent hypergeometric function \cite{magnus1966formulas}. This is a different formula for the Variogram Score of a multivariate (mixture) Gaussian, than is used by a.o. Bjerregård et al. \cite{BJERREGARD2021100058}, and avoids the approximation of an integral.

\section{Algorithms}
The data generating algorithm for Figure 2 in the main paper is described in Algorithm \ref{alg:corr_checking}.
\begin{algorithm}
    \caption{Mean score computation.}\label{alg:corr_checking}
    \begin{algorithmic}
        \Function{Compute\_score}{$R$, $\mu$, $\rho$, $\sigma$, $n = 5000$}
            \State $P_{\text{true}} = \mathcal{N}\left(\begin{pmatrix} 1\\-1\end{pmatrix}, \begin{pmatrix} 1 & 0.8\\ 0.8 & 4\end{pmatrix}\right)$ \hfill\Comment{Define the true data distribution}
            \State Sample $v_1, \ldots, v_n \sim P_\text{true}$. \hfill\Comment{Sample $n$ vectors form the data distribution}
            \State $P = \mathcal{N}\left(\begin{pmatrix} \mu \\-1\end{pmatrix}, \begin{pmatrix} \sigma^2 & 2\rho\sigma\\ 2\rho\sigma & 4\end{pmatrix}\right)$\hfill\Comment{Define the predicted distribution}
            \State\Return $\frac{1}{n} \sum_{i=1}^n R(P, v_i)$\hfill\Comment{Compute the score over the data and prediction}
        \EndFunction
    \end{algorithmic}
\end{algorithm}

\end{document}